\documentclass{article}

\PassOptionsToPackage{numbers, compress}{natbib}

\usepackage[main, final]{neurips_2026}

\usepackage[utf8]{inputenc} 
\usepackage[T1]{fontenc}    
\usepackage{hyperref}       
\usepackage{url}            
\usepackage{booktabs}       
\usepackage{amsfonts}       
\usepackage{nicefrac}       
\usepackage{microtype}      
\usepackage{xcolor}         

\usepackage{amsmath}
\usepackage{amssymb}
\usepackage{mathtools}
\usepackage{amsthm}
\usepackage{bm}

\usepackage{algorithm}
\usepackage{algpseudocode}

\usepackage{makecell}
\usepackage{thmtools}
\usepackage{thm-restate}
\usepackage{graphicx}
\usepackage{subcaption}
\usepackage{multirow}
\usepackage{etoc}

\usepackage{footnotehyper}

\definecolor{darkblue}{rgb}{0,0.08,0.45}
\hypersetup{
    colorlinks=true,
    linkcolor=darkblue,
    citecolor=darkblue,
    urlcolor=darkblue
}

\theoremstyle{plain}
\newtheorem{theorem}{Theorem}[section]
\newtheorem{proposition}[theorem]{Proposition}
\newtheorem{lemma}[theorem]{Lemma}
\newtheorem{corollary}[theorem]{Corollary}
\theoremstyle{definition}

\theoremstyle{remark}

\title{Assessing Per-Sample Membership Inference Vulnerability without Retraining}

\author{%
  Valentin Dorseuil~\thanks{Corresponding author: \texttt{valentin.dorseuil@ens.psl.eu}} \\
  DI ENS, École normale supérieure\\
  Université PSL, CNRS\\
  75005 Paris, France\\
  \And
  Jamal Atif \\
  CMAP, École polytechnique\\
  Institut Polytechnique de Paris\\
  91120 Palaiseau, France\\
  \And
  Olivier Cappé \\
  DI ENS, École normale supérieure\\
  Université PSL, CNRS\\
  75005 Paris, France\\
}

\begin{document}

\maketitle

\begin{abstract}
    Recent work in the privacy literature shows that sample-targeted membership inference attacks (MIAs) significantly outperform untargeted approaches by a wide margin. Motivated by this observation, we address the following question:
    \emph{can the privacy vulnerability of individual training points be assessed without training shadow models?} We show that per-sample exposure to MIA is governed not only by a point's loss, but also by a data-dependent geometric measure.
    In the linear setting, we derive a closed-form decomposition of individual black-box MIA vulnerability into a population leverage score and a residual loss term, making explicit how sample-dependent geometry translates into privacy exposure. Since the final layer of most modern architectures is linear, we extend this framework to deep networks and propose a surrogate score operating on last-layer representations that requires only a single trained model and no shadow models. 
    Empirical evaluations across diverse datasets and architectures show that our score outperforms loss and gradient-norm baselines at identifying the highest-risk points under state-of-the-art attacks, providing a computationally efficient and theoretically grounded tool for per-sample privacy risk assessment.
\end{abstract}

\section{Introduction}
Modern machine learning models, and deep neural networks in particular, are known to memorize aspects of their training data~\citep{zhang2017understanding, carliniSecretSharerEvaluating2019}. This memorization induces privacy vulnerabilities that can be exploited by \emph{Membership Inference Attacks} (MIAs), which aim to determine whether a specific data point was included in the training set or not~\citep{shokriMembershipInferenceAttacks2017,carliniMembershipInferenceAttacks2022}. 
A principled defense against membership inference is provided by Differential Privacy (DP)~\citep{DifferentialPrivacyDwork}, implemented in deep learning through noise injection in stochastic gradient methods~\citep{abadiDeepLearningDifferential2016}. However, controlling the trade-off between privacy protection and model utility remains challenging. Noise calibration typically relies on worst-case theoretical accounting, paired with empirical privacy auditing via MIAs, and often leads to either overly conservative noise levels or insufficient privacy protection. In this context, designing MIAs for empirical auditing is essential to quantify privacy leakage in non-privately trained models~\citep{yeomPrivacyRiskMachine2018} or to validate the practical tightness of DP guarantees in privately trained models~\citep{nasrAdversaryInstantiationLower2021, jagielskiAuditingDifferentiallyPrivate2020}.

While such auditing is now standard practice~\citep{carliniMembershipInferenceAttacks2022, nasrAdversaryInstantiationLower2021, zarifzadehLowCostHighPowerMembership2024}, relying on aggregate metrics such as average accuracy or AUC is insufficient. 
These global measures can obscure critical risk heterogeneity, as outliers and rare subgroups are significantly more prone to memorization than typical samples~\citep{carliniMembershipInferenceAttacks2022,feldman2020neural}. Consequently, a model certified as private on average may still expose specific points to high privacy risks.
To address this heterogeneity, recent work is increasingly focusing on \emph{per-sample} privacy risk assessment, aiming to quantify membership leakage at the level of each data point, separately, rather than in aggregate. State-of-the-art methods for per-sample auditing~\citep{carliniMembershipInferenceAttacks2022,zarifzadehLowCostHighPowerMembership2024} predominantly rely on \emph{shadow models}. 
These techniques train multiple shadow models on random data splits to characterize each data point's influence on the model's behavior. This enables identification of which points exhibit increased sensitivity to their presence in the training dataset. 
However, these techniques are computationally prohibitive, especially for large-scale models, as they require retraining the model many times. This leads to our central question stated in the abstract: \emph{can the privacy vulnerability of individual training points be assessed without training shadow models?}

To address this question in a principled way, we first consider the baseline case of Gaussian linear regression. Even in this simple model, the task poses interesting challenges: the fixed-design analysis from classical statistics leads to the \emph{leverage score} (see below), which quantifies a data point's geometric influence solely through its feature vector, ignoring the response.
The fully randomized approach, on the other hand, provides metrics that depend on general problem characteristics (dimensions, covariance operator of the features) but do not account for specific data points (see, e.g., \citet{tan2022parameters}). 
To obtain relevant per-sample risk exposure metrics, we propose to consider non-asymptotic results in the randomized Gaussian setting by \emph{conditioning on a single data point} (i.e., its feature vector and its response).

For Gaussian linear regression, we derive a closed-form expression for the expected leave-one-out generalization gap, the signal exploited by state-of-the-art black-box attacks such as LiRA \citep{carliniMembershipInferenceAttacks2022}, and show that it decomposes, at first order, into a response-free geometric term, the population leverage score~$\bar{h}_{ii}$, and a response-dependent structural error~$\varepsilon_i$. 
Membership inference vulnerability is therefore governed jointly by a point's geometric measure and its structural error. Samples that are both geometrically influential in shaping the model’s parameters and poorly fit by the model are at the highest risk of privacy leakage.

To extend this analysis beyond the linear setting, we approximate the leave-one-out loss gap using two classical estimators: an influence-function estimate and a Newton-step estimate. Restricted to the last layer of a trained network, both admit closed-form expressions in terms of the leverage score of the learned representations, and can be computed from a single trained model at negligible cost. This yields a scalable, theoretically grounded proxy for per-sample privacy risk that matches the rankings produced by state-of-the-art shadow-model attacks, without any retraining. 

The paper is organized as follows. In Section~\ref{sec:theory} we derive a closed-form expression for the expected leave-one-out generalization gap of a target point in Gaussian linear regression, showing that the signal driving LiRA-style attacks decomposes, at first order, into a label-free geometric term, the population leverage~$\bar{h}_{ii}$, and a label-dependent structural error~$\varepsilon_i$. In Section~\ref{sec:gen_lin}, we propose two retraining-free per-sample vulnerability scores for general differentiable models, based, respectively, on influence-function and Newton-step approximations of the leave-one-out loss. 
Restricted to the last layer, both admit closed-form expressions in terms of the leverage score of the last-layer representation, which can be efficiently computed for deep networks. Finally, Section~\ref{sec:exp} reports the performance of the proposed vulnerability scores in terms of recovery of the results obtained by state-of-the-art retraining-based MIA methods. In particular, we observe that the proposed metrics consistently outperform commonly used baselines, such as training loss or per-sample gradient-norm, across various configurations of dataset/architecture pairs.
We start with a quick review of the relevant literature.

\section{Related Work}

\paragraph{Membership Inference Attacks.} 
Membership inference aims to determine if a specific sample was used to train a model. Early approaches relied on simple metric-based classifiers, exploiting overfitting signals such as prediction confidence, entropy, or the magnitude of gradients~\citep{shokriMembershipInferenceAttacks2017, yeomPrivacyRiskMachine2018}. While computationally inexpensive, these methods often struggle to distinguish between a ``vulnerable'' member and a ``hard'' non-member.

To address this, current state-of-the-art methods adopt the \emph{shadow model} paradigm~\citep{shokriMembershipInferenceAttacks2017}. By training multiple models on different data splits, attacks like LiRA (Likelihood Ratio Attack)~\citep{carliniMembershipInferenceAttacks2022} or RMIA (Robust Membership Inference Attack)~\citep{zarifzadehLowCostHighPowerMembership2024} essentially perform a hypothesis test on the loss distribution of a target point~\citep{zarifzadehLowCostHighPowerMembership2024}. While highly effective, these approaches are computationally prohibitive, often requiring hundreds of training runs to estimate the risk for a single point accurately. Our work seeks to achieve the precision of these hypothesis-test based methods without their computational burden, by substituting retraining with geometric analysis.

\paragraph{Influence Functions and Data Attribution.} 
Recent work has begun to explore the connection between influence and privacy. Notably, \citet{feldman2020neural} and \citet{feldman2020does} link the \emph{memorization} of a sample to its influence on the learning process, arguing that memorization is necessary for generalization in long-tailed distributions. 
Other work has used influence functions to craft new attack methods \citep{cohenMembershipInferenceAttack2024, suri2024parameters}, but these approaches inherit two well-known limitations of influence functions in deep learning: a prohibitive computational cost, and a fragility that often renders the resulting estimates unreliable~\citep{basu2021influence}.

Our work addresses these limitations by confining influence estimation to the last layer, where closed-form expressions are available. This reduces the computational overhead and avoids the fragility of full-network influence, while remaining a faithful proxy of per-sample influence as shown in our experiments.

\paragraph{Privacy Auditing and Heterogeneity.} 
Differential Privacy (DP)~\citep{DifferentialPrivacyDwork} provides worst-case guarantees for membership privacy. 
Using algorithms such as DP-SGD~\citep{abadiDeepLearningDifferential2016}, one can train Deep Learning models to be private. However, these bounds are often loose and do not capture the empirical reality that privacy risk is non-uniform: some outlier samples are far more exposed than others~\citep{zarifzadehLowCostHighPowerMembership2024}. 

Auditing per-sample privacy risk, i.e., estimating the specific ($\varepsilon, \delta$) DP parameters for a given sample, remains an open challenge. Existing auditing tools rely heavily on the aforementioned shadow model techniques or randomized smoothing~\citep{lecuyer2019certified}, which scale poorly. \citet{carliniMembershipInferenceAttacks2022} derive per-sample vulnerability estimates from the outputs of the LiRA attacks, which need training shadow models. \citet{song2021systematic} also propose a per-sample privacy risk score calibrated with shadow models. At the opposite end of the spectrum, \citet{leemann2024is} empirically evaluate a large set of retraining-free predictors of per-sample membership inference success, including the baselines we compare against.

Our work contributes to this domain by proposing an efficient individual privacy-risk metric for identifying vulnerable samples, providing a first-pass, zero-retraining filter for high-risk data points. We use the LiRA-derived vulnerability of \citet{carliniMembershipInferenceAttacks2022} as our main ground truth and aim to predict it without training any shadow models. We also evaluate our metric against additional attacks~\citep{shokriMembershipInferenceAttacks2017, zarifzadehLowCostHighPowerMembership2024} to assess its robustness.

Removing these vulnerable points from the training set is not an option since, in practice, such deletions often expose new points, an effect described as the
\emph{privacy onion effect} by~\citet{carliniPrivacyOnionEffect2022a}. Instead, our method can serve as a tool for canary selection~\citep{carliniSecretSharerEvaluating2019} or per-sample privacy training~\citep{wang2018perinstancedifferentialprivacy,alaggan2015heterogeneousdifferentialprivacy}.

\section{Analysis of Membership Vulnerability in the Linear Regression Case}
\label{sec:theory}

To rigorously understand the privacy leakage exploited by black-box attacks, we first analyze their behavior within the context of linear models.
State-of-the-art MIAs, such as shadow model classifiers~\citep{shokriMembershipInferenceAttacks2017}, LiRA~\citep{carliniMembershipInferenceAttacks2022}, and RMIA~\citep{zarifzadehLowCostHighPowerMembership2024}, compute membership scores via classifiers or likelihood ratio tests fitted to per-sample losses. In practice, they estimate the loss distribution of each sample, conditioned on whether it was included in the training set, by training shadow models on multiple data splits. 

To model the stochasticity of training on random subsamples, we isolate a fixed target point $i$ and treat the remainder of the dataset as being drawn from an underlying probability distribution. We then analyze the marginal distribution of the target's loss under two scenarios: the $\text{IN}$ setting, where the model is trained on the fixed point alongside the samples drawn from the distribution, and the $\text{OUT}$ setting, where the model is trained exclusively on the sampled distribution without the target point.

\subsection{Problem Setup and Notation}

Let $\mathcal{Z} \subset \mathbb{R}^p \times \mathbb{R}$ be the data space, and let $z_i = (x_i, y_i) \in \mathcal{Z}$ be a fixed target observation. Throughout our analysis, we condition on this target point, treating it as a deterministic constant. We consider a full training dataset $S$ of size $n$ that contains this fixed target $z_i$. The other $n-1$ points in $S$, denoted collectively as $S_{-i} = \{(X_j, Y_j)\}_{j \neq i}$, are random variables sampled from the underlying data distribution. 

Because $z_i$ is fixed, the randomness in our estimators arises entirely from these $n-1$ background observations. We assume their covariates are drawn i.i.d. from a multivariate Gaussian distribution, $X_j \sim \mathcal{N}(0, \Sigma)$, and their responses follow the linear generative model $Y_j \mid X_j \sim \mathcal{N}(X_j^\top\theta, \sigma^2)$ where $\theta\in \mathbb{R}^p$ is the unknown true parameter and $\sigma^2$ the structural noise variance. Equivalently, $Y_{-i} = X_{-i}\theta + \varepsilon_{-i}$, where $\varepsilon_{-i}\in\mathbb{R}^{n-1}$ is the noise vector, with $\varepsilon_{-i}\sim\mathcal{N}(0,\sigma^2 I_{n-1})$ independent of $X_{-i}$. We consider the scalar response case for simplicity, and the results are extended to multivariate responses in Appendix~\ref{app:sec3_proofs}.

Let $X \in \mathbb{R}^{n \times p}$ denote the full design matrix, where the $i$-th row is the deterministic $x_i^\top$ and the remaining rows form the random submatrix $X_{-i} \in \mathbb{R}^{(n-1) \times p}$. Similarly, let $Y \in \mathbb{R}^n$ be the response vector containing $y_i$ and the random subvector $Y_{-i} \in \mathbb{R}^{n-1}$.

Let $\hat{\theta} = (X^\top X)^{-1} X^\top Y$ be the ordinary least squares (OLS) estimator trained on the full dataset $S$, and $\hat{\theta}_{-i}$ be the estimator trained on $S_{-i}$. We denote $e_i = y_i - x_i^\top \hat{\theta}$ the residual for the target point under the full model and $e_{-i} = y_i - x_i^\top \hat{\theta}_{-i}$ the residual for the target point under the leave-one-out model. The corresponding squared residuals are the training loss $\ell_i = (y_i - x_i^\top \hat{\theta})^2$ and the leave-one-out loss $\ell_{-i} = (y_i - x_i^\top \hat{\theta}_{-i})^2$. 

We introduce the leverage matrix $H = X (X^\top X)^{-1} X^\top$, and denote $h_{ii} = x_i^\top (X^\top X)^{-1} x_i$ the empirical leverage score of the target point, which is the $i$-th diagonal element of $H$. This score is a random variable due to its dependence on $X_{-i}$, and quantifies the geometric influence of the point on its own fitted value~\citep{regressionOutliers} illustrated by the leave-one-out identity $e_i = (1-h_{ii})\,e_{-i}$~\citep{allenLeaveOneOut}.

We define two deterministic quantities characterizing the target point: its population leverage score $\bar{h}_{ii} = x_i^\top \Sigma^{-1} x_i$, which corresponds to the squared Mahalanobis distance to the feature distribution, and its structural error $\varepsilon_i = y_i - x_i^\top \theta$.

Our goal is to characterize the distribution of the losses $\ell_i$ and $\ell_{-i}$ and see how these distributions relate to the target point's MIA vulnerability.

\subsection{Finite-Sample Distributions of Leverage Score and Loss}

We begin by analyzing the leave-one-out Gram matrix $G_{-i} = X_{-i}^\top X_{-i}$, which follows a Wishart distribution $\mathcal{W}_p(\Sigma, n-1)$~\citep{multivariateStatTheory}. Applying the Sherman–Morrison–Woodbury identity allows us to express both the empirical leverage score $h_{ii}$ and the leave-one-out residual $e_{-i}$ as functions of $G_{-i}$, thereby isolating the geometric randomness of the design from the structural noise. The following lemma makes this decomposition explicit.

\begin{lemma}[Joint Law of Leverage Scores and Residuals]
\label{lem:leverage-residuals}
Let $G_{-i} = X_{-i}^\top X_{-i}$ denote the leave-one-out Gram matrix, which follows a Wishart distribution $\mathcal{W}_p(\Sigma, n-1)$~\citep{multivariateStatTheory}. We define
\begin{equation}
V = \frac{\bar{h}_{ii}}{x_i^\top G_{-i}^{-1} x_i}, 
\qquad 
Z = -\frac{\sqrt{V}}{\sigma \sqrt{\bar{h}_{ii}}} \, x_i^\top G_{-i}^{-1} X_{-i}^\top \varepsilon_{-i}.
\end{equation}
Then, the empirical leverage score $h_{ii}$ and the leave-one-out residual $e_{-i}$ satisfy the identities:
\begin{equation}
h_{ii} = \frac{\bar{h}_{ii}}{V + \bar{h}_{ii}}, 
\qquad 
e_{-i} = \varepsilon_i + \sigma \frac{\sqrt{\bar{h}_{ii}}}{\sqrt{V}} Z.
\end{equation}
Moreover, conditioned on $(x_i, y_i)$, $V \sim \chi^2_{n-p}$ and $Z \sim \mathcal{N}(0,1)$, with $Z \perp\!\!\!\perp V$.
\end{lemma}

The independence of $Z$ and $V$ in Lemma~\ref{lem:leverage-residuals} is an important property: given $X_{-i}$, the standardized prediction error $Z$ is $\mathcal{N}(0,1)$ regardless of the design, hence independent of any function of it, including $V$. Combining the leave-one-out identity $e_i = (1-h_{ii})e_{-i}$~\citep{allenLeaveOneOut}, which gives $\ell_{i} = (1-h_{ii})^2 \ell_{-i}$, with Lemma~\ref{lem:leverage-residuals}, we obtain a full characterization of the joint distribution of the losses at point $i$ in Proposition~\ref{prop:loss_distribution}.

\begin{proposition}[Per-sample Loss Distribution]
\label{prop:loss_distribution}
Under the Gaussian linear regression model, the leave-one-out loss $\ell_{-i}$ and the training loss $\ell_i$ at a point $(x_i, y_i)$ satisfy
\begin{equation}
\ell_{-i} = \left( \varepsilon_i + \sigma \frac{\sqrt{\bar{h}_{ii}}}{\sqrt{V}} Z \right)^2, 
\qquad 
\ell_i = \left( \varepsilon_i + \sigma \frac{\sqrt{\bar{h}_{ii}}}{\sqrt{V}} Z \right)^2 \left( \frac{V}{V + \bar{h}_{ii}} \right)^2,
\end{equation}
where $V$ and $Z$ are defined in Lemma~\ref{lem:leverage-residuals}, satisfying $V \sim \chi^2_{n-p}$ and $Z \sim \mathcal{N}(0,1)$, with $Z\perp\!\!\!\perp V$.
\end{proposition}

\subsection{Per-Point Vulnerability under Optimal Inference}
Using Proposition~\ref{prop:loss_distribution} and the Neyman–Pearson Lemma~\citep{neymanLikelihood}, we derive the optimal black-box attack in closed form. In the large-sample regime $n - p \gg 1$, the variable $V \sim \chi^2_{n-p}$ concentrates sharply around its mean, so we replace it by the constant $n-p$ and Proposition~\ref{prop:loss_distribution} reduces to a deterministic rescaling in distribution: $\ell_i \mid \mathcal{H}_{\text{IN}} = c_i^2 \cdot (\ell_{-i} \mid \mathcal{H}_{\text{OUT}})$, with the scale factor $c_i^2 = ((n-p)/ (n-p + \bar{h}_{ii}))^2 < 1$. 
The likelihood ratio is then monotone in the observed loss, so the Neyman–Pearson optimal test thresholds $\ell_{i, \text{obs}}$ directly.

\begin{corollary}[Optimal MIA Trade-off]
\label{cor:optimal_attack}
Let $F_{i, \text{OUT}}$ denote the CDF of $\ell_{-i}$ under $\mathcal{H}_{\text{OUT}}$. 
For any target false positive rate $\alpha_i \in (0,1)$, the false negative rate of the Neyman–Pearson optimal attack at point $i$ is
\begin{equation}
    \beta_i(\alpha_i) \;=\; 1 - F_{i, \text{OUT}}\!\left(\tfrac{1}{c_i^2}\, F_{i, \text{OUT}}^{-1}(\alpha_i)\right).
\end{equation}
\end{corollary}

The closed-form expression itself follows from a standard Neyman–Pearson argument; the corollary's value lies in clarifying \emph{which} features of the problem dictate vulnerability.
The full trade-off function $\alpha_i \mapsto \beta_i(\alpha_i)$ depends on the data only through $(\bar{h}_{ii}, \varepsilon_i)$ and the global parameters $(n, p, \sigma^2)$. In particular, the entire design matrix $X_{-i}$ enters only through the scalar population leverage $\bar{h}_{ii}$. The pair $(\bar{h}_{ii}, \varepsilon_i)$ thus acts as a per-point sufficient statistic for membership-inference vulnerability.

\subsection{The Expected Generalization Gap}
Alongside the trade-off curve, the expected generalization gap $\mathbb{E}[\Delta\ell_i]$ has been identified as the fundamental theoretical bound on the power of loss-based MIAs~\citep{yeomPrivacyRiskMachine2018}. In our model, it admits a theoretical closed-form and a tractable asymptotic expression that make the dependence on $(\bar h_{ii}, \varepsilon_i)$ explicit.

\begin{proposition}[Expected Generalization Gap]
\label{prop:gap}
For $n-p > 4$, the expected difference between the leave-one-out and training loss for a fixed point $(x_i, y_i)$ is exactly:
\begin{equation}
    \mathbb{E}[\Delta\ell_i] = \mathbb{E}_{V \sim \chi^2_{n-p}} \left[ \left( \varepsilon_i^2 + \sigma^2 \frac{\bar{h}_{ii}}{V} \right) \frac{\bar{h}_{ii} (2V + \bar{h}_{ii})}{(V + \bar{h}_{ii})^2} \right].
\end{equation}

The asymptotic expansion of this gap in powers of $(n-p)^{-1}$ is:
\begin{equation}
    \mathbb{E}[\Delta\ell_i] = \frac{2\varepsilon_i^2 \bar{h}_{ii}}{n-p} + \frac{4\varepsilon_i^2 \bar{h}_{ii} + \bar{h}_{ii}^2(2\sigma^2 - 3\varepsilon_i^2)}{(n-p)^2} + \mathcal{O}\left(\frac{1}{(n-p)^3}\right).
\end{equation}
\end{proposition}

Proposition~\ref{prop:gap} establishes that, to first order, the generalization gap is governed by the interaction between a point's structural error $\varepsilon_i^2$ and its leverage $\bar{h}_{ii}$. Intuitively, the points exhibiting the largest generalization gap are those that are simultaneously outliers in feature space (high $\bar{h}_{ii}$) and difficult to predict (high $\varepsilon_i^2$). Moreover, the gap vanishes at a rate of $1/n$, which confirms that the privacy risk associated with any individual point diminishes as the size of the dataset grows.

\begin{figure}[t]
    \centering
    \includegraphics{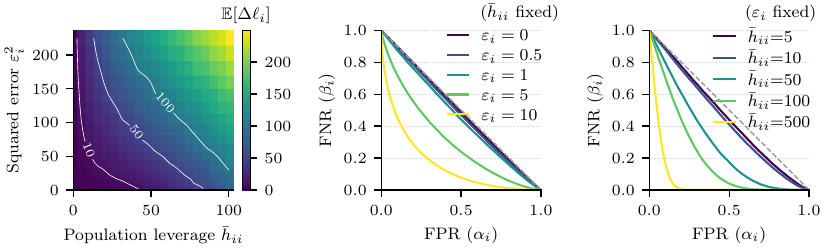}
    \caption{Generalization gap and membership inference attack performance for the Gaussian linear regression model. \textbf{(left)} Heatmap of the generalization gap $\mathbb{E}[\Delta\ell_i]$ as a function of $\varepsilon_i^2$ and $\bar{h}_{ii}$, estimated with Monte-Carlo.  \textbf{(middle)} Theoretical false negative rate (FNR) $\beta_i$ of the optimal attack on the point $i$ as a function of the false positive rate (FPR) $\alpha_i$, for different values of the structural error $\varepsilon_i$, with $\bar{h}_{ii} = 1$ and $\sigma^2 = 1$. \textbf{(right)} Same as (middle) but with fixed $\varepsilon_i = 1$ and $\bar{h}_{ii}$ varying.}
    \label{fig:toy_model}
\end{figure}

Figure~\ref{fig:toy_model} illustrates Corollary~\ref{cor:optimal_attack} and Proposition~\ref{prop:gap} using Monte-Carlo simulations. The left panel shows how the expected generalization gap grows with the true residual $\varepsilon_i^2$ and the leverage score $\bar{h}_{ii}$. The middle and right panels present the trade-off curves of the optimal MIA for varying $\varepsilon_i$ and $\bar{h}_{ii}$, respectively, confirming that the attack's statistical power increases monotonically in both parameters.

Proofs for all the results presented in this section are deferred to Appendix~\ref{app:sec3_proofs}.

\section{Generalization to Non-Linear Models}
\label{sec:gen_lin}
In the linear regression case, we could derive the exact value of the loss gap $\ell_{-i} - \ell_i$ without retraining. This is because the OLS estimator has a closed-form solution and the loss is quadratic. For more general models, such as generalized linear models or deep neural networks, we do not have closed-form solutions. In the following, we discuss how common techniques from influence functions and Newton steps can be used to estimate the loss gap without retraining, and how they relate to the previous analysis in both the linear regression and the logistic regression case.

\subsection{Influence Functions and Newton Step}

\paragraph{Influence Function.}  Influence functions~\citep{hampelInfluence,kohUnderstandingBlackboxPredictions2017} provide a way to estimate the effect of removing a training point $z_i$ on the model parameters $\hat{\theta}$ without retraining. Formally, the influence of $z_i$ is defined as:
\begin{equation*}
I(z_i) = -\Big[\frac{1}{n}\sum_{j=1}^{n} \nabla_\theta^2 \ell(\hat{\theta}, z_j)\Big]^{-1} \nabla_\theta \ell(\hat{\theta}, z_i).
\end{equation*}
Using a first-order Taylor expansion, this defines our first estimator for the loss gap, which we call the \emph{Influence Function estimator}:
\begin{equation}
\label{eq:inf}
\widehat{\Delta\ell}^{\mathrm{IF}}_i = \nabla_\theta \ell(\hat{\theta}, z_i)^\top \Big[\sum_{j=1}^{n} \nabla_\theta^2 \ell(\hat{\theta}, z_j)\Big]^{-1} \nabla_\theta \ell(\hat{\theta}, z_i).
\end{equation}

\paragraph{Newton Step.} Alternatively, one can approximate the loss gap by taking a single Newton step on the leave-one-out objective starting from $\hat{\theta}$. Using a first-order Taylor expansion on the loss, this defines our second estimator, which we call the \emph{Newton Step estimator}:
\begin{equation*}
\widehat{\Delta\ell}^{\mathrm{NS}}_i = - \nabla_\theta \ell(\hat{\theta}, z_i)^\top \Big[\frac{1}{n}\sum_{j\neq i} \nabla_\theta^2 \ell(\hat{\theta}, z_j)\Big]^{-1}  \frac{1}{n}\sum_{j \neq i} \nabla_\theta \ell(\hat{\theta}, z_j).
\end{equation*}
Notably, when the model has fully converged and minimizes the empirical risk (i.e., $\frac{1}{n}\sum_{j=1}^n \nabla_\theta \ell(\hat{\theta}, z_j) = 0$), it follows that $\frac{1}{n}\sum_{j \neq i} \nabla_\theta \ell(\hat{\theta}, z_j) = -\frac{1}{n}\nabla_\theta \ell(\hat{\theta}, z_i)$. Under this convergence assumption, the Newton Step estimator simplifies to:
\begin{equation}
\label{eq:newton}
\widehat{\Delta\ell}^{\mathrm{NS}}_i = \nabla_\theta \ell(\hat{\theta}, z_i)^\top \Big[\sum_{j\neq i} \nabla_\theta^2 \ell(\hat{\theta}, z_j)\Big]^{-1} \nabla_\theta \ell(\hat{\theta}, z_i),
\end{equation}
a form also known as the \emph{Approximate Leave-One-Out} (ALO) estimator in the statistics literature~\citep{rad2020scalable} and used in approximate unlearning methods~\citep{certifiedGuo2020,sekhari2021rememberwantforgetalgorithms}. 

Remarkably, $\widehat{\Delta\ell}^{\mathrm{IF}}_i$ and $\widehat{\Delta\ell}^{\mathrm{NS}}_i$ coincide up to a rank-one correction of the Hessian, despite being derived from distinct principles. In the following section, we derive closed-form expressions for these two estimators in the context of linear and logistic regression, and show how they can serve as scalable proxies for the loss gap to evaluate vulnerability to membership inference attacks.

\subsection{Analysis for Linear Regression and Logistic Regression}

We now derive the exact expressions of $\widehat{\Delta\ell}^{\mathrm{IF}}_i$ and $\widehat{\Delta\ell}^{\mathrm{NS}}_i$ for linear and logistic regression. These are not further approximations: they are the direct, closed-form instantiations of Equations~\eqref{eq:inf} and~\eqref{eq:newton} for the respective model classes.

\paragraph{Linear Regression.}
In linear regression, the OLS estimator $\hat{\theta} = (X^\top X)^{-1} X^\top Y$ is available in closed form and the loss is quadratic, so both estimators admit exact closed-form expressions. Letting $e_i$ denote the residual and $h_{ii}$ the leverage score, the two estimators yield:
\begin{equation}
    \label{eq:lin_reg_inf}
    \widehat{\Delta\ell}^{\mathrm{IF}}_i = 2 e_i^2 h_{ii}, \qquad
    \widehat{\Delta\ell}^{\mathrm{NS}}_i = \frac{2 e_i^2 h_{ii}}{1 - h_{ii}}.
\end{equation}
The Newton Step estimator introduces a leverage correction $(1-h_{ii})^{-1}$, bringing it closer to the exact leave-one-out formula: $ \Delta\ell_i = e_i^2(2h_{ii} - h_{ii}^2)/(1-h_{ii})^2 $, particularly for high-leverage points.

\paragraph{Logistic Regression.}
In binary logistic regression, with $\hat{p}_i = s(\hat{\theta}^\top x_i)$, the two estimators yield:
\begin{equation}
    \label{eq:log_reg_inf}
    \widehat{\Delta\ell}^{\mathrm{IF}}_i = \frac{(y_{i}-\hat{p}_{i})^{2}}{w_{i}}\, h_{ii},
    \qquad
    \widehat{\Delta\ell}^{\mathrm{NS}}_i = \frac{(y_{i}-\hat{p}_{i})^{2}}{w_{i}}\, \frac{h_{ii}}{1 - h_{ii}},
\end{equation}
where $w_i = \hat{p}_i (1 - \hat{p}_i)$, $s$ is the sigmoid function and $h_{ii} = w_i x_i^\top (X^\top W X)^{-1} x_i$ is the weighted leverage score, with $W = \operatorname{diag}(w_j)_{j=1,\dots,n}$. This is the logistic regression analogue of the OLS leverage score~\citep{pregibonLogReg}: in both cases, $h_{ii}$ captures the local sensitivity of the fitted response to the target, $\partial \hat{y}_i / \partial y_i$ in linear regression and $\partial \hat{p}_i / \partial y_i$ in logistic regression. As in Equation~\eqref{eq:lin_reg_inf}, the Newton Step estimator applies a leverage correction relative to the Influence Function estimator.

We generalize both of these results for the multivariate case in Appendix~\ref{app:sec4_multivariate}. Motivated by these findings, as well as our theoretical analysis in Proposition~\ref{prop:gap}, we propose to use these estimates as proxies for vulnerability assessment in deep neural networks. In practice, we consider only the linear last layer of the network, and apply the estimators defined in equations~\eqref{eq:lin_reg_inf} and~\eqref{eq:log_reg_inf} (or their multivariate counterparts) depending on the loss used for the training (cross-entropy for classification, quadratic for regression). 

We emphasize that these estimates are designed to evaluate a point's inherent vulnerability (its outlier status within the dataset distribution, by approximating the loss gap) and not its vulnerability to one specific model instance\footnote{A point's vulnerability is defined as how strongly it leaks across training randomness. Evaluating this vulnerability with respect to a single model instance alone is ill-defined, as one (point, model) pair yields a single attack outcome.}.

\subsection{Differences Between the Two Approaches}

In both linear and logistic regression, the two estimators differ by a single 
multiplicative factor: 
$\widehat{\Delta\ell}^{\mathrm{NS}}_i / \widehat{\Delta\ell}^{\mathrm{IF}}_i 
= 1/(1-h_{ii})$. At low leverage they therefore agree to leading order. The 
correction matters only as $h_{ii} \to 1$, a regime reached, for instance, when $x_i$ 
is orthogonal to every other row of $X$, in which case $h_{ii} = 1$ exactly. We have: 
\begin{equation*}
    \frac{\widehat{\Delta\ell}^{\mathrm{IF}}_i}{\Delta\ell_i}  \underset{h_{ii}\to 1}{\sim} 2(1-h_{ii})^2,
    \qquad
    \frac{\widehat{\Delta\ell}^{\mathrm{NS}}_i}{\Delta\ell_i} \underset{h_{ii}\to 1}{\sim} 2(1-h_{ii}).
\end{equation*}

As $h_{ii} \to 1$, both estimators vanish relative to $\Delta\ell_i$, but at different rates: the Influence Function estimator captures only an $O((1-h_{ii})^2)$ fraction of the true loss change, whereas the Newton Step estimator captures an $O(1-h_{ii})$ fraction. The Newton Step therefore tracks the per-sample generalization gap substantially better than the Influence Function at high-leverage points.

Proofs for all the results presented in this section are deferred to Appendix~\ref{app:sec4_proofs}.

\section{Experiments}
\label{sec:exp}
We evaluate our two surrogates $\widehat{\Delta\ell}^{\mathrm{IF}}_i$ and $\widehat{\Delta\ell}^{\mathrm{NS}}_i$ for per-sample membership-inference vulnerability across a diverse set of modalities, architectures and learning regimes.

\subsection{Experimental Setup}
\label{sec:exp_setup}

\begin{figure}[tbp]
    \centering
    \includegraphics{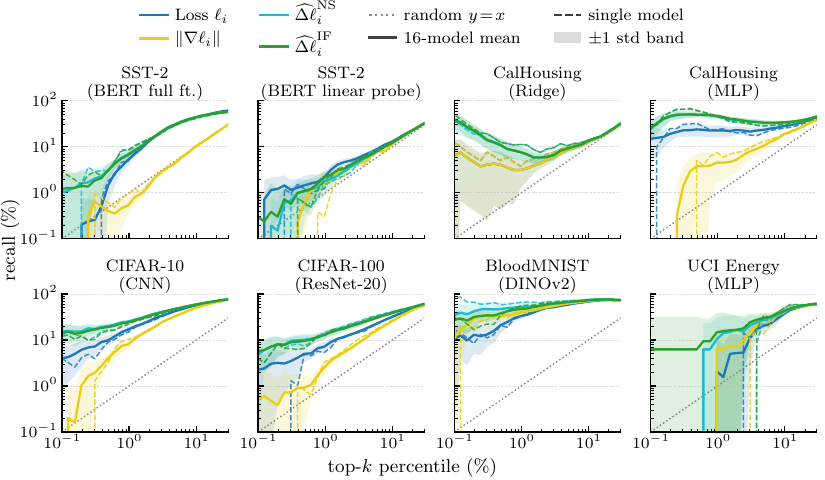} 
    \caption{Recall of the $k$\% of points most vulnerable to LiRA (ranked by the per-sample Attack Success Rate) when keeping the top-$k$\% for each of the four surrogates: the loss $\ell_i$, the gradient norm $\|\nabla\ell_i\|$, the influence estimate $\widehat{\Delta\ell}^{\mathrm{IF}}_i$, and the Newton estimate $\widehat{\Delta\ell}^{\mathrm{NS}}_i$. Axes are log-log; the dotted diagonal $y=x$ is the random baseline. Solid lines are means over $16$ target models with $\pm 1$ std bands; dashed-lines show a single representative model.}
    \label{fig:topk}
\end{figure}

\paragraph{Datasets and architectures.} We cover both classification and regression on tabular, vision and text data:
(i) \emph{CIFAR-10}~\citep{cifar10} with a small CNN;
(ii) \emph{CIFAR-100}~\citep{cifar10} with ResNet-20~\citep{resnet};
(iii) \emph{BloodMNIST}~\citep{bloodmnist,medmnistv1,medmnistv2} with a DINOv2 ViT-small classifier~\citep{oquab2023dinov2}, trained with a linear probe on top of frozen features;
(iv) \emph{SST-2}~\citep{sst2} with a BERT~\citep{devlin2019bert} classifier in two regimes: full fine-tuning of \texttt{bert\_uncased\_L-2\_H-128\_A-2}~\citep{turc2019wellreadstudentslearnbetter} and a linear probe on top of a frozen \texttt{bert-base-uncased};
(v) \emph{California Housing}~\citep{californiaHousing} with both a ridge regression and a deep MLP;
(vi) \emph{UCI Energy}~\citep{uciEnergy} with a tabular MLP. 
In Appendix~\ref{app:resnet} we also conduct experiments on a larger dataset, with a variant as we use the K-Fac~\citep{martensKfac} approximation for the computation of the inverse-Hessian-vector products.
For each task we use the standard train split, randomly partitioned in half between members and non-members ($N_\text{train}=N/2$) with a per-model seed, following the LiRA shadow-model protocol~\citep{carliniMembershipInferenceAttacks2022}.

\paragraph{Shadow and target models.} For every dataset/architecture pair we train $200$ shadow (reference) models with i.i.d. random member/non-member splits, plus $16$ independent target models used to compute the per-sample vulnerability surrogates. 
Optimizers, learning rates, batch sizes, weight decay and number of epochs follow standard recipes for each architecture (AdamW for vision and text, SGD-momentum for ResNet-20, Adam for the regression MLPs); the exact hyper-parameters are listed in Appendix~\ref{app:arch_hyperparams} and in the release scripts.

\paragraph{Attacks and surrogates.} We use LiRA~\citep{carliniMembershipInferenceAttacks2022} as the reference per-sample attack, since it achieves the highest Attack Success Rate (ASR) in our experiments. As our ground truth, we choose the per-sample ASR computed over the 200 reference models, rather than the TPR@low-FPR. As a \emph{per-sample} metric, the latter depends on the extreme tails of the shadow model loss distribution, and is therefore noisier than the per-sample ASR\footnote{This conclusion applies specifically to estimating the vulnerability of individual samples. As an aggregate metric, the TPR@low-FPR is often preferred, as it is more informative about the leakage of the most vulnerable points~\citep{carliniMembershipInferenceAttacks2022}. For completeness, we also compared with the TPR@low-FPR and provide the results in Appendix~\ref{app:tpr_low_fpr}.}. 

We compare four \emph{retraining-free} surrogates of vulnerability, all computed from a \emph{single} target model and \emph{without} shadow models: the training loss $\ell_i$, the gradient-norm $\|\nabla\ell_i\|$, the influence-function estimate $\widehat{\Delta\ell}^{\mathrm{IF}}_i$ and the Newton-step estimate $\widehat{\Delta\ell}^{\mathrm{NS}}_i$. All experiments were performed on points in the training set (\emph{members}). The same quantities can be computed from the loss observed on \emph{non-members} by reversing the equations; see Appendix~\ref{app:non_members} for details. 

For deep networks, both variants are restricted to the last layer, which keeps their cost essentially negligible. In Appendix~\ref{app:add_exp_results}, we extend this analysis to include comparisons with both the RMIA attack~\citep{zarifzadehLowCostHighPowerMembership2024} and the original shadow-model attack~\citep{shokriMembershipInferenceAttacks2017}. For two settings, CIFAR-10 with a CNN and California Housing with an MLP, we additionally computed the full-network influence function estimate of the loss, without the restriction to the last layer, i.e., $\nabla_\theta \ell_i^\top H_{\hat\theta}^{-1}\nabla_\theta \ell_i$. We compute it using the inverse- -Hessian-vector Products \citep{hessianProduct} and the Conjugate Gradient algorithm \citep{hestenesMethodsConjugateGradients1952}. More details can be found in Appendix~\ref{app:full_inf}.

\paragraph{Metrics.} 
Following standard per-sample auditing practice~\citep{shokriMembershipInferenceAttacks2017, carliniMembershipInferenceAttacks2022, zarifzadehLowCostHighPowerMembership2024}, we evaluate surrogates based on top-$k$\% recall: the fraction of the most vulnerable points (highest ASR) correctly identified within the surrogate's top-$k$\% rankings. The Spearman rank correlation is provided in Appendix~\ref{app:lira_full_results} for reference, but we find it less informative in our context as it is dominated by the high density of points near the 50\% ASR (random guess), where surrogates provide little discriminative signal.

All experiments were run on an NVIDIA V100 GPU cluster. The full pipeline (training, statistics, attacks, scores and figures) is reproducible from the released code available at \url{https://github.com/vdorseuil/per_sample_membership}.
    
\subsection{Results}

\begin{table}[t]
\centering
\small
\renewcommand{\arraystretch}{1.2}
\setlength{\tabcolsep}{4pt}
\caption{Recall (\%) of the true LiRA top-1\,\% when keeping the top-5\,\% by each surrogate. Mean $\pm$ std across target models. Higher is better; \textbf{bold} = best per dataset.}
\label{tab:proxy-recall-at-5pct-lira}
\begin{tabular}{ll c c c c}
\toprule
& & \multicolumn{4}{c}{Recall of Top 1\% @5\% (in \%)} \\
\cmidrule(lr){3-6}
Dataset & Model & Loss $\ell_i$ & $\|\nabla\ell_i\|$ & $\widehat{\Delta\ell}^{\mathrm{NS}}_i$ & $\widehat{\Delta\ell}^{\mathrm{IF}}_i$ \\
\midrule
\multirow{2}{*}{SST-2} & BERT, full ft. & $28.2 \pm 2.4$\,\% & $5.3 \pm 1.1$\,\% & $28.8 \pm 2.6$\,\% & $\mathbf{28.8} \pm \mathbf{2.6}$\,\% \\
& BERT, linear probe & $\mathbf{10.7} \pm \mathbf{1.6}$\,\% & $7.2 \pm 1.9$\,\% & $8.2 \pm 2.5$\,\% & $8.9 \pm 1.6$\,\% \\

\multirow{2}{*}{CalHousing} & Ridge regression & $12.3 \pm 2.8$\,\% & $12.3 \pm 2.8$\,\% & $\mathbf{13.3} \pm \mathbf{2.6}$\,\% & $13.2 \pm 2.5$\,\% \\
& MLP & $52.2 \pm 6.7$\,\% & $22.5 \pm 6.9$\,\% & $74.9 \pm 7.4$\,\% & $\mathbf{75.5} \pm \mathbf{7.0}$\,\% \\

CIFAR-10 & CNN & $53.7 \pm 2.5$\,\% & $36.8 \pm 2.5$\,\% & $\mathbf{62.8} \pm \mathbf{2.1}$\,\% & $62.2 \pm 2.2$\,\% \\

CIFAR-100 & ResNet-20 & $23.5 \pm 1.5$\,\% & $12.3 \pm 1.8$\,\% & $\mathbf{36.8} \pm \mathbf{5.1}$\,\% & $35.2 \pm 4.6$\,\% \\

BloodMNIST & DINOv2 & $74.3 \pm 6.4$\,\% & $82.3 \pm 3.9$\,\% & $\mathbf{90.4} \pm \mathbf{12.7}$\,\% & $87.1 \pm 9.0$\,\% \\

UCI Energy & MLP & $12.5 \pm 16.7$\,\% & $14.6 \pm 21.0$\,\% & $33.3 \pm 34.4$\,\% & $\mathbf{39.6} \pm \mathbf{32.7}$\,\% \\
\bottomrule
\end{tabular}
\end{table}

We present the recall @5\% of the true LiRA top-1\% (measured by per-sample Attack Success Rate) in Table~\ref{tab:proxy-recall-at-5pct-lira}. Figure~\ref{fig:topk} illustrates the performance of these surrogates for multiple recall percentages\footnote{At very small top-$k$ percentiles, the sample size of considered points can be extremely low. Consequently, the random baseline represents a very small expected value. Due to variance across just 16 random seeds, it is highly probable to observe a lower recovery rate than the mathematical expectation. This explains the drop of the recall ($y$-axis) for low top-$k$ percentiles ($x$-axis).}. We also present more extensive results for different percentages of recall in Appendix~\ref{app:add_exp_results}, which show similar trends. 

Across seven out of eight settings,  our two surrogates $\widehat{\Delta\ell}^{\mathrm{IF}}_i$ and $\widehat{\Delta\ell}^{\mathrm{NS}}_i$, consistently recover the samples LiRA flags as most vulnerable, well above the random baseline, often by one to two orders of magnitude. The results are more mixed on BERT with linear probing. We suspect this arises from the fact that the model learned very little during the training process, as the training loss did not decrease significantly.

They also surpass the loss baseline on seven out of eight settings, with recall@$1\%$ improving by more than $20\%$ on some datasets. The sole exception is the BERT linear probe, where the loss baseline marginally outperforms our method; we hypothesize that this setting, characterized by simple training dynamics and few vulnerable points, is one in which the raw loss already provides a strong signal. We further observe high variance in recall@$1\%$ across all surrogates on the UCI Energy dataset, which we attribute to the mismatch between the small dataset size ($N=768$) and the high-dimensional tabular MLP ($p=512$). At this scale, recall@$1\%$ is evaluated over only 7 points, making the metric inherently noisy.

\subsection{Discussion}

Across settings, our leverage-based surrogates have greater difficulty recovering the points most vulnerable to RMIA and classifier-based attacks than those vulnerable to LiRA, specifically in the regression case (Appendix~\ref{app:rmia_shokri_full_results}). This follows from our construction: in the linear model, our surrogate is built as an unbiased estimator of the loss gap, measuring the difference in loss when the target point is included in versus excluded from the dataset. This quantity is precisely the fundamental signal governing the LiRA attack, which approximates the likelihood ratio of the loss under these two hypotheses. In contrast, the RMIA attack estimates a pairwise likelihood ratio that captures how the loss changes when the target point is replaced by another one, rather than removed. Likewise, the classifier-based approach learns a non-parametric decision boundary that might not be well captured by the leverage-based scores. This contrast explains the stronger recovery of points most vulnerable to LiRA. Despite these differences, our method still outperforms the loss baseline on RMIA in all cases, with results nearly matching those on LiRA in the classification case. Finally, the full-network influence function proved worse in both speed and recall compared to our last-layer approach (see Appendix~\ref{app:full_inf}).

Overall, these scores yield a shadow-free, single-model proxy that approaches LiRA's per-point ASR ranking at a fraction of its computational cost. They can be used by model developers or privacy auditors to detect the most vulnerable points easily. Scoring all points with our leverage-based surrogates on a single target model is $1{,}000\times$ to $10{,}000\times$ faster than training the 200 shadow models required by LiRA, excluding the target-model training cost, which is shared by both approaches. Detailed timings for each surrogate and attack are reported in the Appendix \ref{app:compute_time}.

\section{Conclusion}
We proposed a retraining-free framework for per-sample membership inference vulnerability assessment. In Gaussian linear regression, we derived a closed-form decomposition of the expected leave-one-out gap into a response-free population leverage term and a response-dependent structural error. We extended this view to deep networks via influence-function and Newton-step surrogates restricted to the last layer, both computable from a single trained model. Across seven out of eight dataset/architecture pairs, our surrogates outperform loss and gradient-norm baselines at recovering LiRA's most vulnerable points, at a fraction of the cost. 

Limitations remain: the theory is exact only in the linear case, the last-layer restriction overlooks earlier feature learning, and the surrogates track LiRA more closely than RMIA or classifier-based attacks. Natural extensions include layer-wise or kernel-based leverage scores, integration into DP-SGD training as a canary-selection or per-sample weighting tool, and a finer characterization of which memorization signals escape local quadratic approximations.

\section*{Acknowledgements}
This work is part of the project SSF-ML-DH of the PEPR Digital Health and was supported by the French National Research Agency, as a part of the France 2030 program, under grant ANR-22-PESN-0014-01. Furthermore, we thank GENCI at IDRIS for providing computer and storage resources through the grant 2026-AD011017108 on the supercomputer Jean Zay.

We are also grateful to Hanna Benarroch for her helpful feedback on the project.

\bibliographystyle{unsrtnat}
\bibliography{paper}

\begin{thebibliography}{49}
\providecommand{\natexlab}[1]{#1}
\providecommand{\url}[1]{\texttt{#1}}
\expandafter\ifx\csname urlstyle\endcsname\relax
  \providecommand{\doi}[1]{doi: #1}\else
  \providecommand{\doi}{doi: \begingroup \urlstyle{rm}\Url}\fi

\bibitem[Zhang et~al.(2017)Zhang, Bengio, Hardt, Recht, and
  Vinyals]{zhang2017understanding}
Chiyuan Zhang, Samy Bengio, Moritz Hardt, Benjamin Recht, and Oriol Vinyals.
\newblock Understanding deep learning requires rethinking generalization.
\newblock In \emph{International Conference on Learning Representations}, 2017.

\bibitem[Carlini et~al.(2019)Carlini, Liu, Erlingsson, Kos, and
  Song]{carliniSecretSharerEvaluating2019}
Nicholas Carlini, Chang Liu, {\'U}lfar Erlingsson, Jernej Kos, and Dawn Song.
\newblock The secret sharer: Evaluating and testing unintended memorization in
  neural networks.
\newblock In \emph{USENIX Security Symposium}, pages 267--284, 2019.

\bibitem[Shokri et~al.(2017)Shokri, Stronati, Song, and
  Shmatikov]{shokriMembershipInferenceAttacks2017}
Reza Shokri, Marco Stronati, Congzheng Song, and Vitaly Shmatikov.
\newblock Membership inference attacks against machine learning models.
\newblock In \emph{IEEE Symposium on Security and Privacy (SP)}, pages 3--18,
  2017.

\bibitem[Carlini et~al.(2022{\natexlab{a}})Carlini, Chien, Nasr, Song, Terzis,
  and Tram{\`e}r]{carliniMembershipInferenceAttacks2022}
Nicholas Carlini, Steve Chien, Milad Nasr, Shuang Song, Andreas Terzis, and
  Florian Tram{\`e}r.
\newblock Membership inference attacks from first principles.
\newblock In \emph{IEEE Symposium on Security and Privacy (SP)}, pages
  1897--1914, 2022{\natexlab{a}}.

\bibitem[Dwork(2006)]{DifferentialPrivacyDwork}
Cynthia Dwork.
\newblock Differential privacy.
\newblock In \emph{International Colloquium on Automata, Languages and
  Programming (ICALP)}, pages 1--12, 2006.

\bibitem[Abadi et~al.(2016)Abadi, Chu, Goodfellow, McMahan, Mironov, Talwar,
  and Zhang]{abadiDeepLearningDifferential2016}
Mart{\'i}n Abadi, Andy Chu, Ian Goodfellow, H.~Brendan McMahan, Ilya Mironov,
  Kunal Talwar, and Li~Zhang.
\newblock Deep learning with differential privacy.
\newblock In \emph{ACM SIGSAC Conference on Computer and Communications
  Security (CCS)}, pages 308--318, 2016.

\bibitem[Yeom et~al.(2018)Yeom, Giacomelli, Fredrikson, and
  Jha]{yeomPrivacyRiskMachine2018}
Samuel Yeom, Irene Giacomelli, Matt Fredrikson, and Somesh Jha.
\newblock Privacy risk in machine learning: Analyzing the connection to
  overfitting.
\newblock In \emph{IEEE Computer Security Foundations Symposium (CSF)}, pages
  268--282, 2018.

\bibitem[Nasr et~al.(2021)Nasr, Song, Thakurta, Papernot, and
  Carlini]{nasrAdversaryInstantiationLower2021}
Milad Nasr, Shuang Song, Abhradeep Thakurta, Nicolas Papernot, and Nicholas
  Carlini.
\newblock Adversary instantiation: Lower bounds for differentially private
  machine learning.
\newblock In \emph{IEEE Symposium on Security and Privacy (SP)}, pages
  866--882, 2021.

\bibitem[Jagielski et~al.(2020)Jagielski, Ullman, and
  Oprea]{jagielskiAuditingDifferentiallyPrivate2020}
Matthew Jagielski, Jonathan Ullman, and Alina Oprea.
\newblock Auditing differentially private machine learning: How private is
  private {SGD}?
\newblock In \emph{Advances in Neural Information Processing Systems}, pages
  22205--22216, 2020.

\bibitem[Zarifzadeh et~al.(2024)Zarifzadeh, Liu, and
  Shokri]{zarifzadehLowCostHighPowerMembership2024}
Sajjad Zarifzadeh, Philippe Liu, and Reza Shokri.
\newblock Low-cost high-power membership inference attacks.
\newblock In \emph{International Conference on Machine Learning}, pages
  58244--58282, 2024.

\bibitem[Feldman and Zhang(2020)]{feldman2020neural}
Vitaly Feldman and Chiyuan Zhang.
\newblock What neural networks memorize and why: Discovering the long tail via
  influence estimation.
\newblock In \emph{Advances in Neural Information Processing Systems}, pages
  2881--2891, 2020.

\bibitem[Tan et~al.(2022)Tan, Mason, Javadi, and Baraniuk]{tan2022parameters}
Jasper Tan, Blake Mason, Hamid Javadi, and Richard Baraniuk.
\newblock Parameters or privacy: A provable tradeoff between
  overparameterization and membership inference.
\newblock In \emph{Advances in Neural Information Processing Systems}, pages
  17488--17500, 2022.

\bibitem[Feldman(2020)]{feldman2020does}
Vitaly Feldman.
\newblock Does learning require memorization? a short tale about a long tail.
\newblock In \emph{ACM SIGACT Symposium on Theory of Computing (STOC)}, pages
  954--959, 2020.

\bibitem[Cohen and Giryes(2024)]{cohenMembershipInferenceAttack2024}
Gilad Cohen and Raja Giryes.
\newblock Membership inference attack using self influence functions.
\newblock In \emph{IEEE/CVF Winter Conference on Applications of Computer
  Vision (WACV)}, pages 4880--4889, 2024.

\bibitem[Suri et~al.(2024)Suri, Zhang, and Evans]{suri2024parameters}
Anshuman Suri, Xiao Zhang, and David Evans.
\newblock Do parameters reveal more than loss for membership inference?
\newblock \emph{Transactions on Machine Learning Research}, 2024.

\bibitem[Basu et~al.(2021)Basu, Pope, and Feizi]{basu2021influence}
Samyadeep Basu, Phillip Pope, and Soheil Feizi.
\newblock Influence functions in deep learning are fragile.
\newblock In \emph{International Conference on Learning Representations}, 2021.

\bibitem[Lecuyer et~al.(2019)Lecuyer, Atlidakis, Geambasu, Hsu, and
  Jana]{lecuyer2019certified}
Mathias Lecuyer, Vaggelis Atlidakis, Roxana Geambasu, Daniel Hsu, and Suman
  Jana.
\newblock Certified robustness to adversarial examples with differential
  privacy.
\newblock In \emph{IEEE Symposium on Security and Privacy (SP)}, pages
  656--672, 2019.

\bibitem[Song and Mittal(2021)]{song2021systematic}
Liwei Song and Prateek Mittal.
\newblock Systematic evaluation of privacy risks of machine learning models.
\newblock In \emph{USENIX Security Symposium}, pages 2615--2632, 2021.

\bibitem[Leemann et~al.(2024)Leemann, Prenkaj, and Kasneci]{leemann2024is}
Tobias Leemann, Bardh Prenkaj, and Gjergji Kasneci.
\newblock Is my data safe? predicting instance-level membership inference
  success for white-box and black-box attacks.
\newblock In \emph{ICML 2024 Next Generation of AI Safety Workshop}, 2024.

\bibitem[Carlini et~al.(2022{\natexlab{b}})Carlini, Jagielski, Zhang, Papernot,
  Terzis, and Tram{\`e}r]{carliniPrivacyOnionEffect2022a}
Nicholas Carlini, Matthew Jagielski, Chiyuan Zhang, Nicolas Papernot, Andreas
  Terzis, and Florian Tram{\`e}r.
\newblock The privacy onion effect: Memorization is relative.
\newblock In \emph{Advances in Neural Information Processing Systems}, pages
  13263--13276, 2022{\natexlab{b}}.

\bibitem[Wang(2019)]{wang2018perinstancedifferentialprivacy}
Yu-Xiang Wang.
\newblock Per-instance differential privacy.
\newblock \emph{Journal of Privacy and Confidentiality}, 9\penalty0 (1), 2019.

\bibitem[Alaggan et~al.(2017)Alaggan, Gambs, and
  Kermarrec]{alaggan2015heterogeneousdifferentialprivacy}
Mohammad Alaggan, S{\'e}bastien Gambs, and Anne-Marie Kermarrec.
\newblock Heterogeneous differential privacy.
\newblock \emph{Journal of Privacy and Confidentiality}, 7\penalty0 (2), 2017.

\bibitem[Belsley et~al.(1980)Belsley, Kuh, and Welsch]{regressionOutliers}
David~A. Belsley, Edwin Kuh, and Roy~E. Welsch.
\newblock Detecting influential observations and outliers.
\newblock In \emph{Regression Diagnostics: Identifying Influential Data and
  Sources of Collinearity}, chapter~2, pages 6--84. John Wiley \& Sons, 1980.

\bibitem[Allen(1974)]{allenLeaveOneOut}
David~M. Allen.
\newblock The relationship between variable selection and data agumentation and
  a method for prediction.
\newblock \emph{Technometrics}, 16\penalty0 (1):\penalty0 125--127, 1974.

\bibitem[Muirhead(1982)]{multivariateStatTheory}
Robb~J. Muirhead.
\newblock The multivariate linear model.
\newblock In \emph{Aspects of Multivariate Statistical Theory}, chapter~10,
  pages 429--525. John Wiley \& Sons, 1982.

\bibitem[Neyman and Pearson(1933)]{neymanLikelihood}
Jerzy Neyman and Egon~Sharpe Pearson.
\newblock On the problem of the most efficient tests of statistical hypotheses.
\newblock \emph{Philosophical Transactions of the Royal Society of London.
  Series A}, 231:\penalty0 289--337, 1933.

\bibitem[Hampel(1974)]{hampelInfluence}
Frank~R. Hampel.
\newblock The influence curve and its role in robust estimation.
\newblock \emph{Journal of the American Statistical Association}, 69\penalty0
  (346):\penalty0 383--393, 1974.

\bibitem[Koh and Liang(2017)]{kohUnderstandingBlackboxPredictions2017}
Pang~Wei Koh and Percy Liang.
\newblock Understanding black-box predictions via influence functions.
\newblock In \emph{International Conference on Machine Learning}, pages
  1885--1894, 2017.

\bibitem[Rad and Maleki(2020)]{rad2020scalable}
Kamiar~Rahnama Rad and Arian Maleki.
\newblock A scalable estimate of the out-of-sample prediction error via
  approximate leave-one-out cross-validation.
\newblock \emph{Journal of the Royal Statistical Society Series B (Statistical
  Methodology)}, 82\penalty0 (4):\penalty0 965--996, 2020.

\bibitem[Guo et~al.(2020)Guo, Goldstein, Hannun, and van~der
  Maaten]{certifiedGuo2020}
Chuan Guo, Tom Goldstein, Awni Hannun, and Laurens van~der Maaten.
\newblock Certified data removal from machine learning models.
\newblock In \emph{International Conference on Machine Learning}, pages
  3832--3842, 2020.

\bibitem[Sekhari et~al.(2021)Sekhari, Acharya, Kamath, and
  Suresh]{sekhari2021rememberwantforgetalgorithms}
Ayush Sekhari, Jayadev Acharya, Gautam Kamath, and Ananda~Theertha Suresh.
\newblock Remember what you want to forget: Algorithms for machine unlearning.
\newblock In \emph{Advances in Neural Information Processing Systems}, pages
  18075--18086, 2021.

\bibitem[Pregibon(1981)]{pregibonLogReg}
Daryl Pregibon.
\newblock Logistic regression diagnostics.
\newblock \emph{The Annals of Statistics}, 9\penalty0 (4):\penalty0 705--724,
  1981.

\bibitem[Krizhevsky(2009)]{cifar10}
Alex Krizhevsky.
\newblock Learning multiple layers of features from tiny images.
\newblock Technical report, University of Toronto, 2009.

\bibitem[He et~al.(2016)He, Zhang, Ren, and Sun]{resnet}
Kaiming He, Xiangyu Zhang, Shaoqing Ren, and Jian Sun.
\newblock Deep residual learning for image recognition.
\newblock In \emph{IEEE Conference on Computer Vision and Pattern Recognition
  (CVPR)}, pages 770--778, 2016.

\bibitem[Acevedo et~al.(2020)Acevedo, Merino, Alférez, Ángel Molina, Boldú,
  and Rodellar]{bloodmnist}
Andrea Acevedo, Anna Merino, Santiago Alférez, Ángel Molina, Laura Boldú,
  and José Rodellar.
\newblock A dataset of microscopic peripheral blood cell images for development
  of automatic recognition systems.
\newblock \emph{Data in Brief}, 30:\penalty0 105474, 2020.

\bibitem[Yang et~al.(2021)Yang, Shi, and Ni]{medmnistv1}
Jiancheng Yang, Rui Shi, and Bingbing Ni.
\newblock Medmnist classification decathlon: A lightweight automl benchmark for
  medical image analysis.
\newblock In \emph{IEEE 18th International Symposium on Biomedical Imaging
  (ISBI)}, pages 191--195, 2021.

\bibitem[Yang et~al.(2023)Yang, Shi, Wei, Liu, Zhao, Ke, Pfister, and
  Ni]{medmnistv2}
Jiancheng Yang, Rui Shi, Donglai Wei, Zequan Liu, Lin Zhao, Bilian Ke,
  Hanspeter Pfister, and Bingbing Ni.
\newblock Medmnist v2-a large-scale lightweight benchmark for 2d and 3d
  biomedical image classification.
\newblock \emph{Scientific Data}, 10\penalty0 (1):\penalty0 41, 2023.

\bibitem[Oquab et~al.(2024)Oquab, Darcet, Moutakanni, Vo, Szafraniec, Khalidov,
  Fernandez, Haziza, Massa, El-Nouby, Assran, Ballas, Galuba, Howes, Huang, Li,
  Misra, Rabbat, Sharma, Synnaeve, Xu, Jegou, Mairal, Labatut, Joulin, and
  Bojanowski]{oquab2023dinov2}
Maxime Oquab, Timoth{\'e}e Darcet, Th{\'e}o Moutakanni, Huy Vo, Marc
  Szafraniec, Vasil Khalidov, Pierre Fernandez, Daniel Haziza, Francisco Massa,
  Alaaeldin El-Nouby, Mahmoud Assran, Nicolas Ballas, Wojciech Galuba, Russell
  Howes, Po-Yao Huang, Shang-Wen Li, Ishan Misra, Michael Rabbat, Vasu Sharma,
  Gabriel Synnaeve, Hu~Xu, Herv{\'e} Jegou, Julien Mairal, Patrick Labatut,
  Armand Joulin, and Piotr Bojanowski.
\newblock {DINOv2}: Learning robust visual features without supervision.
\newblock \emph{Transactions on Machine Learning Research}, 2024.

\bibitem[Socher et~al.(2013)Socher, Perelygin, Wu, Chuang, Manning, Ng, and
  Potts]{sst2}
Richard Socher, Alex Perelygin, Jean Wu, Jason Chuang, Christopher~D. Manning,
  Andrew Ng, and Christopher Potts.
\newblock Recursive deep models for semantic compositionality over a sentiment
  treebank.
\newblock In \emph{Conference on Empirical Methods in Natural Language
  Processing}, pages 1631--1642, 2013.

\bibitem[Devlin et~al.(2019)Devlin, Chang, Lee, and Toutanova]{devlin2019bert}
Jacob Devlin, Ming-Wei Chang, Kenton Lee, and Kristina Toutanova.
\newblock {BERT}: Pre-training of deep bidirectional transformers for language
  understanding.
\newblock In \emph{North American Chapter of the Association for Computational
  Linguistics: Human Language Technologies (NAACL)}, pages 4171--4186, 2019.

\bibitem[Turc et~al.(2019)Turc, Chang, Lee, and
  Toutanova]{turc2019wellreadstudentslearnbetter}
Iulia Turc, Ming-Wei Chang, Kenton Lee, and Kristina Toutanova.
\newblock Well-read students learn better: On the importance of pre-training
  compact models.
\newblock arXiv:1908.08962, 2019.

\bibitem[{Kelley Pace} and Barry(1997)]{californiaHousing}
R.~{Kelley Pace} and Ronald Barry.
\newblock Sparse spatial autoregressions.
\newblock \emph{Statistics \& Probability Letters}, 33\penalty0 (3):\penalty0
  291--297, 1997.

\bibitem[Tsanas and Xifara(2012)]{uciEnergy}
Athanasios Tsanas and Angeliki Xifara.
\newblock Accurate quantitative estimation of energy performance of residential
  buildings using statistical machine learning tools.
\newblock \emph{Energy and Buildings}, 49:\penalty0 560--567, 2012.

\bibitem[Martens and Grosse(2015)]{martensKfac}
James Martens and Roger Grosse.
\newblock Optimizing neural networks with kronecker-factored approximate
  curvature.
\newblock In \emph{International Conference on Machine Learning}, pages
  2408--2417, 2015.

\bibitem[Pearlmutter(1994)]{hessianProduct}
Barak~A. Pearlmutter.
\newblock Fast exact multiplication by the hessian.
\newblock \emph{Neural Computation}, 6\penalty0 (1):\penalty0 147--160, 1994.

\bibitem[Hestenes and Stiefel(1952)]{hestenesMethodsConjugateGradients1952}
Magnus~R. Hestenes and Eduard Stiefel.
\newblock Methods of conjugate gradients for solving linear systems.
\newblock \emph{Journal of Research of the National Bureau of Standards},
  49\penalty0 (6):\penalty0 409--436, 1952.

\bibitem[Deng et~al.(2009)Deng, Dong, Socher, Li, Li, and
  Fei-Fei]{dengImagenet}
Jia Deng, Wei Dong, Richard Socher, Li-Jia Li, Kai Li, and Li~Fei-Fei.
\newblock Imagenet: A large-scale hierarchical image database.
\newblock In \emph{IEEE Conference on Computer Vision and Pattern Recognition
  (CVPR)}, pages 248--255, 2009.

\bibitem[Salem et~al.(2019)Salem, Zhang, Humbert, Berrang, Fritz, and
  Backes]{salem2018ml}
Ahmed Salem, Yang Zhang, Mathias Humbert, Pascal Berrang, Mario Fritz, and
  Michael Backes.
\newblock {ML}-leaks: Model and data independent membership inference attacks
  and defenses on machine learning models.
\newblock In \emph{Network and Distributed System Security Symposium (NDSS)},
  2019.

\bibitem[Breiman(2001)]{breimanRandomForests2001}
Leo Breiman.
\newblock Random forests.
\newblock \emph{Machine Learning}, 45\penalty0 (1):\penalty0 5--32, 2001.

\end{thebibliography}

\newpage
\appendix

\addcontentsline{toc}{part}{Appendix}

\etocsettocstyle{\section*{Table of Contents}}{}
\etocsetnexttocdepth{subsection}
\etocstandardlines
\localtableofcontents

\section{Omitted Proofs from Section~\ref{sec:theory}}
\label{app:sec3_proofs}

This appendix provides the rigorous proofs for the theoretical results characterizing the finite-sample generalization gap. Throughout these proofs, we maintain the conditioning on the fixed target observation $(x_i, y_i)$, treating $x_i$ and $y_i$ as deterministic constants. The randomness originates entirely from $X_{-i}$ and $Y_{-i}$.

\subsection{Proof of Lemma~\ref{lem:leverage-residuals} (Joint Law of Leverage Scores and Residuals)}
\begin{proof}
We first derive the distribution of the empirical leverage score $h_{ii}$. With $G_{-i} = X_{-i}^\top X_{-i}$, the full Gram matrix  can be decomposed as $X^\top X = G_{-i} + x_i x_i^\top$. The empirical leverage score for the $i$-th observation is $h_{ii} = x_i^\top (X^\top X)^{-1} x_i$. 

Applying the Sherman-Morrison formula to invert the rank-one update:
\begin{equation*}
    (G_{-i} + x_i x_i^\top)^{-1} = G_{-i}^{-1} - \frac{G_{-i}^{-1} x_i x_i^\top G_{-i}^{-1}}{1 + x_i^\top G_{-i}^{-1} x_i}
\end{equation*}
Multiplying both sides by $x_i^\top$ on the left and $x_i$ on the right yields:
\begin{equation*}
    h_{ii} = x_i^\top G_{-i}^{-1} x_i - \frac{(x_i^\top G_{-i}^{-1} x_i)^2}{1 + x_i^\top G_{-i}^{-1} x_i}
\end{equation*}
Let $\tilde{h}_{ii} = x_i^\top G_{-i}^{-1} x_i$. Simplifying the fraction, we obtain $h_{ii} = \frac{\tilde{h}_{ii}}{1 + \tilde{h}_{ii}}$.

Because the rows of $X_{-i}$ are i.i.d. $\mathcal{N}(0, \Sigma)$, $G_{-i}$ follows a Wishart distribution $\mathcal{W}_p(\Sigma, n-1)$. A standard property of the Inverse-Wishart distribution states that for any fixed vector $x_i$, the quadratic form scales as an inverse chi-squared random variable. Specifically:
\begin{equation*}
    \frac{x_i^\top \Sigma^{-1} x_i}{x_i^\top G_{-i}^{-1} x_i} \sim \chi^2_{(n-1) - p + 1} = \chi^2_{n-p}
\end{equation*}
Let this random variable be denoted as $V$. Recalling that $\bar{h}_{ii} = x_i^\top \Sigma^{-1} x_i$, we have $V = \bar{h}_{ii} / \tilde{h}_{ii}$, which implies $\tilde{h}_{ii} = \bar{h}_{ii} / V$. Substituting this into our expression for $h_{ii}$ gives:
\begin{equation*}
    h_{ii} = \frac{\bar{h}_{ii} / V}{1 + \bar{h}_{ii} / V} = \frac{\bar{h}_{ii}}{V + \bar{h}_{ii}}
\end{equation*}
which concludes the first part of the lemma. 

We now derive the distribution of the leave-one-out residual $e_{-i}$. The leave-one-out OLS estimator is $\hat{\theta}_{-i} = G_{-i}^{-1} X_{-i}^\top Y_{-i}$. Under the true model, $Y_{-i} = X_{-i}\theta + \varepsilon_{-i}$, where $\varepsilon_{-i} \sim \mathcal{N}(0, \sigma^2 I_{n-1})$ is unconditionally independent of $X_{-i}$.
Substituting this yields $\hat{\theta}_{-i} = \theta + G_{-i}^{-1} X_{-i}^\top \varepsilon_{-i}$.

The leave-one-out residual is:
\begin{equation*}
    e_{-i} = y_i - x_i^\top \hat{\theta}_{-i} = (y_i - x_i^\top \theta) - x_i^\top G_{-i}^{-1} X_{-i}^\top \varepsilon_{-i} = \varepsilon_i - U
\end{equation*}
where $\varepsilon_i = y_i - x_i^\top \theta$ is a deterministic constant and $U = x_i^\top G_{-i}^{-1} X_{-i}^\top \varepsilon_{-i}$.

We now condition on $X_{-i}$. Since $\varepsilon_{-i} \perp\!\!\!\perp X_{-i}$, the conditional distribution of $\varepsilon_{-i}$ remains $\mathcal{N}(0, \sigma^2 I_{n-1})$. Because $U$ is a linear combination of a Gaussian vector, $U \mid X_{-i}$ is Gaussian. Its conditional moments are:
\begin{align}
    \mathbb{E}[U \mid X_{-i}] &= x_i^\top G_{-i}^{-1} X_{-i}^\top \mathbb{E}[\varepsilon_{-i}] = 0 \\
    \text{Var}(U \mid X_{-i}) &= x_i^\top G_{-i}^{-1} X_{-i}^\top (\sigma^2 I_{n-1}) X_{-i} G_{-i}^{-1} x_i = \sigma^2 x_i^\top G_{-i}^{-1} G_{-i} G_{-i}^{-1} x_i = \sigma^2 \tilde{h}_{ii}
\end{align}
Substituting $\tilde{h}_{ii} = \bar{h}_{ii} / V$, we have $U \mid X_{-i} \sim \mathcal{N}(0, \sigma^2 \bar{h}_{ii} / V)$. Let $Z = -U / (\sigma \sqrt{\bar{h}_{ii}} / \sqrt{V})$. Then $Z \mid X_{-i} \sim \mathcal{N}(0, 1)$. 

Crucially, because the conditional distribution of $Z$ given $X_{-i}$ is standard normal and does not depend on $X_{-i}$, $Z$ is unconditionally independent of $X_{-i}$. Since $V$ is entirely a function of $X_{-i}$, it strictly follows that $Z \perp\!\!\!\perp V$.
Thus, we can write the residual as $e_{-i} = \varepsilon_i + \sigma \frac{\sqrt{\bar{h}_{ii}}}{\sqrt{V}} Z$.

Finally, by standard regression update rules, the training residual is $e_i = e_{-i} / (1 + \tilde{h}_{ii})$. Squaring this yields the training loss:
\begin{equation*}
    \ell_i = e_i^2 = \left( \frac{e_{-i}}{1 + \bar{h}_{ii} / V} \right)^2 = e_{-i}^2 \left( \frac{V}{V + \bar{h}_{ii}} \right)^2
\end{equation*}
\end{proof}

\subsection{Proof of Proposition~\ref{prop:loss_distribution} (Per-sample Loss Distribution)}
\begin{proof}
The results follow directly from Lemma~\ref{lem:leverage-residuals} by writing that the training loss is 
\[
\ell_i = e_i^2 = e_{-i}^2 (1-h_{ii})^2
\] and the leave-one-out loss is $\ell_{-i} = e_{-i}^2$.

\end{proof}

\subsection{Proof of Corollary~\ref{cor:optimal_attack} (Optimal MIA Trade-off)}
\begin{proof}

Under the non-member hypothesis $\mathcal{H}_{\text{OUT}}$, the optimal decision rule rejects the null hypothesis if the observed loss is unusually small: $\ell_{\text{obs}, i} < \gamma_\alpha$. For a fixed False Positive Rate $\alpha$, the threshold is strictly determined by the non-member cumulative distribution function: $\alpha = \mathbb{P}(\ell_{\text{obs}, i} < \gamma_\alpha \mid \mathcal{H}_{\text{OUT}})$, yielding $\gamma_\alpha = F_{i, \text{OUT}}^{-1}(\alpha)$.

Under the member hypothesis $\mathcal{H}_{\text{IN}}$, the loss is scaled by $c_i^2$. The True Positive Rate is the probability that this scaled loss falls below the same threshold:
\begin{equation*}
    \text{TPR}(\alpha) = \mathbb{P}\left(c_i^2 \ell_{\text{obs}, i} < \gamma_\alpha \mid \mathcal{H}_{\text{OUT}}\right) = \mathbb{P}\left(\ell_{\text{obs}, i} < \frac{\gamma_\alpha}{c_i^2} \mid \mathcal{H}_{\text{OUT}}\right) = F_{i, \text{OUT}}\left( \frac{\gamma_\alpha}{c_i^2} \right)
\end{equation*}

We finish the proof by writing $\beta_i(\alpha) = \text{FNR}(\alpha) = 1- \text{TPR}(\alpha)$
\end{proof}

\subsection{Proof of Proposition~\ref{prop:gap} (Expected Generalization Gap)}
\begin{proof}
The generalization gap is $\Delta \ell_i =   \ell_{-i} - \ell_i$. Using the relationship $\ell_i = e_{-i}^2 [V / (V + \bar{h}_{ii})]^2$, we can factor out $e_{-i}^2$:
\begin{equation*}
    \Delta \ell_i = e_{-i}^2 \left( 1 - \frac{V^2}{(V + \bar{h}_{ii})^2} \right) = e_{-i}^2 \left( \frac{(V + \bar{h}_{ii})^2 - V^2}{(V + \bar{h}_{ii})^2} \right) = e_{-i}^2 \frac{\bar{h}_{ii}(2V + \bar{h}_{ii})}{(V + \bar{h}_{ii})^2}
\end{equation*}
To compute the expectation $\mathbb{E}[\Delta \ell_i \mid x_i, y_i]$, we use the independence of $Z$ and $V$ established in Lemma~\ref{lem:leverage-residuals}. We first take the inner expectation with respect to $Z$:
\begin{align}
    \mathbb{E}_Z[e_{-i}^2 \mid V] &= \mathbb{E}_Z \left[ \left( \varepsilon_i + \sigma \frac{\sqrt{\bar{h}_{ii}}}{\sqrt{V}} Z \right)^2 \right] \\
    &= \varepsilon_i^2 + 2\varepsilon_i\sigma\frac{\sqrt{\bar{h}_{ii}}}{\sqrt{V}}\mathbb{E}[Z] + \sigma^2\frac{\bar{h}_{ii}}{V}\mathbb{E}[Z^2] = \varepsilon_i^2 + \sigma^2\frac{\bar{h}_{ii}}{V}
\end{align}
Substituting this inner expectation into the gap yields the exact finite-sample formula:
\begin{equation*}
    \mathbb{E}[\Delta \ell_i \mid x_i, y_i] = \mathbb{E}_{V \sim \chi^2_{n-p}} \left[ \left( \varepsilon_i^2 + \sigma^2 \frac{\bar{h}_{ii}}{V} \right) \frac{\bar{h}_{ii} (2V + \bar{h}_{ii})}{(V + \bar{h}_{ii})^2} \right]
\end{equation*}

To derive the asymptotic expansion, let $k = n - p$ and define $x = 1/V$. We expand the rational term as $x \to 0$ using the Maclaurin series for $(1 + \bar{h}_{ii} x)^{-2}$:
\begin{equation*}
    1 - (1 + \bar{h}_{ii} x)^{-2} = 1 - (1 - 2\bar{h}_{ii} x + 3\bar{h}_{ii}^2 x^2 - \dots) = 2\bar{h}_{ii} x - 3\bar{h}_{ii}^2 x^2 + \mathcal{O}(x^3)
\end{equation*}
Multiplying by the expanded residual term $\mathbb{E}_Z[e_{-i}^2] = \varepsilon_i^2 + \sigma^2 \bar{h}_{ii} x$:
\begin{align}
    g(x) &= (\varepsilon_i^2 + \sigma^2 \bar{h}_{ii} x)(2\bar{h}_{ii} x - 3\bar{h}_{ii}^2 x^2 + \mathcal{O}(x^3)) \nonumber \\
    &= 2\varepsilon_i^2 \bar{h}_{ii} x + \bar{h}_{ii}^2(2\sigma^2 - 3\varepsilon_i^2) x^2 + \mathcal{O}(x^3)
\end{align}
We now apply the exact inverse moments of the chi-squared distribution $V \sim \chi^2_k$. For $k > 4$:
\begin{equation*}
    \mathbb{E}[x] = \mathbb{E}[V^{-1}] = \frac{1}{k-2}, \quad \mathbb{E}[x^2] = \mathbb{E}[V^{-2}] = \frac{1}{(k-2)(k-4)}
\end{equation*}
Substituting these moments into our expansion:
\begin{equation*}
    \mathbb{E}[\Delta \ell_i] = \frac{2\varepsilon_i^2 \bar{h}_{ii}}{k-2} + \frac{\bar{h}_{ii}^2(2\sigma^2 - 3\varepsilon_i^2)}{(k-2)(k-4)} + \mathcal{O}\left(k^{-3}\right)
\end{equation*}
Finally, we convert this falling-factorial series into standard powers of $1/k$ using geometric series expansions: $(k-2)^{-1} = k^{-1} + 2k^{-2} + \mathcal{O}(k^{-3})$ and $(k-2)^{-1}(k-4)^{-1} = k^{-2} + \mathcal{O}(k^{-3})$.
Gathering like terms gives:
\begin{align}
    \mathbb{E}[\Delta \ell_i] &= 2\varepsilon_i^2 \bar{h}_{ii} \left( \frac{1}{k} + \frac{2}{k^2} \right) + \bar{h}_{ii}^2(2\sigma^2 - 3\varepsilon_i^2)\left( \frac{1}{k^2} \right) + \mathcal{O}\left( \frac{1}{k^3} \right) \nonumber \\
    &= \frac{2\varepsilon_i^2 \bar{h}_{ii}}{k} + \frac{4\varepsilon_i^2 \bar{h}_{ii} + \bar{h}_{ii}^2(2\sigma^2 - 3\varepsilon_i^2)}{k^2} + \mathcal{O}\left( \frac{1}{k^3} \right)
\end{align}
Substituting $k = n - p$ completes the proof.
\end{proof}

\subsection{Generalization to the Multivariate Case}
\label{app:sec3_multivariate}

We now extend Proposition~\ref{prop:gap} to a $q$-dimensional response. Let $Y_j \in \mathbb{R}^q$ be generated as $Y_j \mid X_j \sim \mathcal{N}(\Theta^\top X_j, \Sigma_y)$ with $\Theta \in \mathbb{R}^{p \times q}$ and $\Sigma_y \succ 0$. We collect responses in $Y \in \mathbb{R}^{n \times q}$, fit the OLS estimator $\hat{\Theta} = (X^\top X)^{-1} X^\top Y$, and define the training and leave-one-out losses
\[
\ell_i = \|y_i - \hat{\Theta}^\top x_i\|^2, \qquad \ell_{-i} = \|y_i - \hat{\Theta}_{-i}^\top x_i\|^2.
\]
The structural error is now a vector $\varepsilon_i = y_i - \Theta^\top x_i \in \mathbb{R}^q$.

\begin{proposition}[Multivariate generalization gap]
\label{prop:gap_multi}
For $n - p > 4$, conditional on $(x_i, y_i)$,
\begin{equation}
\mathbb{E}[\Delta\ell_i \mid x_i, y_i] = \mathbb{E}_{V \sim \chi^2_{n-p}}\!\left[ \left(\|\varepsilon_i\|^2 + \mathrm{tr}(\Sigma_y)\,\frac{\bar{h}_{ii}}{V}\right) \frac{\bar{h}_{ii}(2V + \bar{h}_{ii})}{(V + \bar{h}_{ii})^2} \right],
\end{equation}
with the asymptotic expansion in $k = n - p$
\begin{equation}
\mathbb{E}[\Delta\ell_i \mid x_i, y_i] = \frac{2\|\varepsilon_i\|^2 \bar{h}_{ii}}{n-p} + \frac{4\|\varepsilon_i\|^2 \bar{h}_{ii} + \bar{h}_{ii}^2\bigl(2\,\mathrm{tr}(\Sigma_y) - 3\|\varepsilon_i\|^2\bigr)}{(n-p)^2} + \mathcal{O}\!\left(\frac{1}{(n-p)^3}\right).
\end{equation}
\end{proposition}

\begin{proof}
The leverage analysis of Lemma~\ref{lem:leverage-residuals} depends only on $X$ and is unchanged: $h_{ii} = \bar{h}_{ii}/(V + \bar{h}_{ii})$ with $V \sim \chi^2_{n-p}$ and $\bar{h}_{ii} = x_i^\top \Sigma^{-1} x_i$.

Writing $E = Y - X\Theta \in \mathbb{R}^{n \times q}$ for the noise matrix (rows i.i.d.\ $\mathcal{N}(0, \Sigma_y)$), we have $\hat{\Theta}_{-i} = \Theta + G_{-i}^{-1} X_{-i}^\top E_{-i}$, hence
\[
e_{-i} = y_i - \hat{\Theta}_{-i}^\top x_i = \varepsilon_i - U, \qquad U = E_{-i}^\top X_{-i} G_{-i}^{-1} x_i \in \mathbb{R}^q.
\]
Conditional on $X_{-i}$, $U$ is a linear combination of the independent rows of $E_{-i}$, hence Gaussian with $\mathbb{E}[U \mid X_{-i}] = 0$ and
\[
\mathrm{Cov}(U \mid X_{-i}) = \bigl(x_i^\top G_{-i}^{-1} X_{-i}^\top X_{-i} G_{-i}^{-1} x_i\bigr)\, \Sigma_y = \tilde h_{ii}\, \Sigma_y = \frac{\bar{h}_{ii}}{V}\,\Sigma_y,
\]
using $\tilde h_{ii} = \bar{h}_{ii} / V$ from the proof of Lemma~\ref{lem:leverage-residuals}. Setting $Z = -\sqrt{V}/\sqrt{\bar{h}_{ii}} \cdot U$, we obtain $Z \mid X_{-i} \sim \mathcal{N}(0, \Sigma_y)$. Since this conditional law does not depend on $X_{-i}$, $Z$ is unconditionally independent of $X_{-i}$ and therefore of $V$. We thus write
\[
e_{-i} = \varepsilon_i + \frac{\sqrt{\bar{h}_{ii}}}{\sqrt{V}}\, Z, \qquad Z \sim \mathcal{N}(0, \Sigma_y), \quad V \sim \chi^2_{n-p}, \quad Z \perp\!\!\!\perp V.
\]
The univariate Sherman-Morrison contraction $e_i = (1 - h_{ii})\, e_{-i}$ applies column-wise to $\hat{\Theta}$, so it carries over to vector residuals and yields $\ell_i = (1 - h_{ii})^2\, \ell_{-i}$.

Combining the two displays,
\[
\Delta\ell_i = \ell_{-i}\bigl(1 - (1 - h_{ii})^2\bigr) = \|e_{-i}\|^2 \cdot \frac{\bar{h}_{ii}(2V + \bar{h}_{ii})}{(V + \bar{h}_{ii})^2}.
\]
By independence of $Z$ and $V$,
\[
\mathbb{E}_Z\bigl[\|e_{-i}\|^2 \,\big|\, V\bigr] = \|\varepsilon_i\|^2 + 2\,\frac{\sqrt{\bar{h}_{ii}}}{\sqrt V}\,\varepsilon_i^\top \mathbb{E}[Z] + \frac{\bar{h}_{ii}}{V}\,\mathbb{E}[Z^\top Z] = \|\varepsilon_i\|^2 + \mathrm{tr}(\Sigma_y)\,\frac{\bar{h}_{ii}}{V},
\]
which yields the exact finite-sample expression in the proposition. The asymptotic expansion follows the same algebra as the proof of Proposition~\ref{prop:gap}, replacing $\varepsilon_i^2$ by $\|\varepsilon_i\|^2$ and $\sigma^2$ by $\mathrm{tr}(\Sigma_y)$.

\end{proof}

The univariate result of Proposition~\ref{prop:gap} is recovered when $q = 1$ and $\Sigma_y = \sigma^2$. Crucially, the leverage $\bar{h}_{ii}$ remains label-free and so retains the same first-order role; only the structural-noise contribution is replaced by the sum of marginal output variances $\mathrm{tr}(\Sigma_y)$.

\section{Omitted Proofs from Section~\ref{sec:gen_lin}}
\label{app:sec4_proofs}

\subsection{Proof of the Influence Function Estimate (Eq.~\eqref{eq:inf})}

We recall here the derivation of the influence function estimate, following the perturbation-based approach popularized by \citet{kohUnderstandingBlackboxPredictions2017}.

\begin{proof}

Let the empirical risk over the full dataset be $R(\theta) = \frac{1}{n} \sum_{j=1}^n \ell(\theta, z_j)$, and let $\hat{\theta} = \arg\min_\theta R(\theta)$ be the empirical risk minimizer.

Rather than strictly removing the $i$-th point, we consider the effect of infinitesimally up-weighting the loss of $z_i$ by some small continuous value $\epsilon$. We define the perturbed empirical risk minimizer as:

\[ \hat{\theta}_{\epsilon, z_i} \triangleq \arg\min_{\theta} \left( \frac{1}{n} \sum_{j=1}^n \ell(\theta, z_j) + \epsilon \ell(\theta, z_i) \right) \]

By the first-order optimality condition, the gradient of the perturbed objective at $\hat{\theta}_{\epsilon, z_i}$ must be zero:

\[ 0 = \frac{1}{n} \sum_{j=1}^n \nabla_\theta \ell(\hat{\theta}_{\epsilon, z_i}, z_j) + \epsilon \nabla_\theta \ell(\hat{\theta}_{\epsilon, z_i}, z_i) \]

We then differentiate this entire condition with respect to $\epsilon$ and evaluate it at $\epsilon = 0$ (where $\hat{\theta}_{\epsilon, z_i} = \hat{\theta}$). Applying the chain rule yields:

\[ 0 = \left[ \frac{1}{n} \sum_{j=1}^n \nabla_\theta^2 \ell(\hat{\theta}, z_j) \right] \left. \frac{d\hat{\theta}_{\epsilon, z_i}}{d\epsilon} \right|_{\epsilon=0} + \nabla_\theta \ell(\hat{\theta}, z_i) \]

Let $H = \frac{1}{n} \sum_{j=1}^n \nabla_\theta^2 \ell(\hat{\theta}, z_j)$ be the Hessian of the empirical risk. Solving for the derivative gives us the influence of up-weighting $z_i$ on the parameters:

\[ \mathcal{I}_{\text{up,params}}(z_i) \triangleq \left. \frac{d\hat{\theta}_{\epsilon, z_i}}{d\epsilon} \right|_{\epsilon=0} = -H^{-1} \nabla_\theta \ell(\hat{\theta}, z_i) \]

Removing the point $z_i$ from the dataset corresponds to reducing its weight from $\frac{1}{n}$ to $0$. We can approximate this by setting $\epsilon = -\frac{1}{n}$. Using a first-order Taylor expansion around $\epsilon = 0$, the change in the parameters is:

\[ \hat{\theta}_{-i} - \hat{\theta} \approx \left. \frac{d\hat{\theta}_{\epsilon, z_i}}{d\epsilon} \right|_{\epsilon=0} \cdot \left(-\frac{1}{n}\right) = \frac{1}{n} H^{-1} \nabla_\theta \ell(\hat{\theta}, z_i) \]

Finally, to estimate the loss gap for the $i$-th point, we apply a first-order Taylor expansion to the loss function itself:

\[ \ell_{-i} - \ell_i\approx \nabla_\theta \ell(\hat{\theta}, z_i)^\top (\hat{\theta}_{-i} - \hat{\theta}) \approx \frac{1}{n} \nabla_\theta \ell(\hat{\theta}, z_i)^\top H^{-1} \nabla_\theta \ell(\hat{\theta}, z_i). \]

Expanding the Hessian $H$ and pulling the factor $n$ out of the inverse gives the final influence function estimate:
\begin{equation*}
    \Delta \ell_i \approx \nabla_\theta \ell(\hat{\theta}, z_i)^\top \Big[\sum_{j=1}^n \nabla_\theta^2 \ell(\hat{\theta}, z_j) \Big]^{-1} \nabla_\theta \ell(\hat{\theta}, z_i)
\end{equation*}
This gives us our influence function estimate for the loss gap $\widehat{\Delta\ell}^{\mathrm{IF}}_i$ when removing the $i$-th point from the training set.
\end{proof}

\subsection{Proof of the Newton Step Estimate (Eq.~\eqref{eq:newton})}

\begin{proof}
Let $\hat{\theta}$ be the empirical risk minimizer and define the leave-one-out objective
\begin{equation*}
\mathcal{L}_{-i}(\theta) = \frac{1}{n}\sum_{j\neq i}\ell(\theta,z_j).
\end{equation*}
Write $g_{-i} = \nabla_\theta \mathcal{L}_{-i}(\hat\theta) = \frac{1}{n}\sum_{j\neq i}\nabla_\theta\ell(\hat\theta,z_j)$ and
$H_{-i} = \nabla^2_\theta \mathcal{L}_{-i}(\hat\theta) = \frac{1}{n}\sum_{j\neq i}\nabla^2_\theta\ell(\hat\theta,z_j)$, and assume $H_{-i}$ is invertible.

The leave-one-out minimizer $\hat\theta_{-i}=\arg\min_\theta \mathcal{L}_{-i}(\theta)$ is approximated by \emph{one Newton step} on $\mathcal{L}_{-i}$ started at $\hat\theta$:
\begin{equation*}
\hat\theta_{-i}-\hat\theta \;\approx\; -H_{-i}^{-1}\, g_{-i}.
\end{equation*}

The loss gap is $\Delta\ell_i = \ell(\hat\theta_{-i},z_i)-\ell(\hat\theta,z_i)$. A first-order Taylor expansion of $\ell(\cdot,z_i)$ around $\hat\theta$ gives
\begin{equation*}
\Delta\ell_i \;\approx\; \nabla_\theta\ell(\hat\theta,z_i)^\top(\hat\theta_{-i}-\hat\theta)
\;\approx\; -\nabla_\theta\ell(\hat\theta,z_i)^\top H_{-i}^{-1}\, g_{-i}.
\end{equation*}
This is the Newton Step estimator
\begin{equation*}
\widehat{\Delta\ell}^{\mathrm{NS}}_i = - \nabla_\theta \ell(\hat{\theta}, z_i)^\top \Big[\frac{1}{n}\sum_{j\neq i} \nabla_\theta^2 \ell(\hat{\theta}, z_j)\Big]^{-1} \frac{1}{n}\sum_{j \neq i} \nabla_\theta \ell(\hat{\theta}, z_j).
\end{equation*}

When the model has fully converged and minimizes the empirical risk, it follows that $\frac{1}{n}\sum_{j \neq i} \nabla_\theta \ell(\hat{\theta}, z_j) = -\frac{1}{n}\nabla_\theta \ell(\hat{\theta}, z_i)$, hence
\begin{equation*}
g_{-i} = \frac1n\sum_{j\neq i}\nabla_\theta\ell(\hat\theta,z_j) = -\frac1n\nabla_\theta\ell(\hat\theta,z_i).
\end{equation*}
Substituting, gives
\begin{equation*}
\widehat{\Delta\ell}^{\mathrm{NS}}_i = \nabla_\theta\ell(\hat\theta,z_i)^\top\Big[\sum_{j\neq i}\nabla^2_\theta\ell(\hat\theta,z_j)\Big]^{-1}\nabla_\theta\ell(\hat\theta,z_i),
\end{equation*}
recovering Eq.~\eqref{eq:newton}.
\end{proof}
This estimate differs from the influence function only through the leave-one-out Hessian, which accounts for the removal of $z_i$'s own curvature (a rank-one correction), and is exact for quadratic losses up to the first-order expansion of $\ell(\cdot,z_i)$.

\subsection{Results in the Linear Regression (Eq.~\eqref{eq:lin_reg_inf})}

We specialize Eq.~\eqref{eq:inf} and Eq.~\eqref{eq:newton} to the squared loss $\ell(\theta, z_j) = (y_j - x_j^\top \theta)^2$ on the dataset $\{(x_j, y_j)\}_{j=1}^n$. Throughout, $e_i = y_i - x_i^\top \hat{\theta}$ denotes the training residual and $h_{ii} = x_i^\top (X^\top X)^{-1} x_i$ the leverage score.

\begin{proof}
For each $j$,
\[
\nabla_\theta \ell(\theta, z_j) = -2(y_j - x_j^\top \theta)\, x_j, \qquad \nabla_\theta^2 \ell(\theta, z_j) = 2\, x_j x_j^\top.
\]
Evaluated at $\hat{\theta}$, $\nabla_\theta \ell(\hat{\theta}, z_i) = -2 e_i\, x_i$.
\paragraph{Influence.} The full-data Hessian is $\sum_{j=1}^n \nabla_\theta^2 \ell(\hat{\theta}, z_j) = 2\, X^\top X$. Substituting in Eq.~\eqref{eq:inf}:
\[
\widehat{\Delta\ell}^{\mathrm{IF}}_i = (-2 e_i x_i)^\top (2\, X^\top X)^{-1} (-2 e_i x_i) = 2\, e_i^2\, x_i^\top (X^\top X)^{-1} x_i = 2\, e_i^2\, h_{ii}.
\]

\paragraph{Newton.} The leave-one-out Hessian is $\sum_{j \neq i} \nabla_\theta^2 \ell(\hat{\theta}, z_j) = 2\, (X^\top X - x_i x_i^\top) = 2\, G_{-i}$. Sherman-Morrison gives
\[
x_i^\top G_{-i}^{-1} x_i = \frac{x_i^\top (X^\top X)^{-1} x_i}{1 - x_i^\top (X^\top X)^{-1} x_i} = \frac{h_{ii}}{1 - h_{ii}}.
\]
Substituting in Eq.~\eqref{eq:newton}:
\[
\widehat{\Delta\ell}^{\mathrm{NS}}_i = (-2 e_i x_i)^\top (2\, G_{-i})^{-1} (-2 e_i x_i) = 2\, e_i^2 \cdot \frac{h_{ii}}{1 - h_{ii}}.
\]
The Newton estimate is exact up to the first-order Taylor expansion of the loss; for the quadratic OLS loss, the second-order remainder is itself $e_i^2 \, h_{ii}^2/(1-h_{ii})^2$, recovering the exact leave-one-out loss gap.
\end{proof}

\subsection{Results in the Logistic Regression (Eq.~\eqref{eq:log_reg_inf})}

We specialize Eq.~\eqref{eq:inf} and Eq.~\eqref{eq:newton} to the binary cross-entropy loss with $\hat{p}_j = s(x_j^\top \hat{\theta})$ and $w_j = \hat{p}_j (1 - \hat{p}_j)$. Let $W = \mathrm{diag}(w_j)_{j=1}^n$ and define the  leverage score for the logistic model~\citep{pregibonLogReg} as:
\[
h_{ii} = w_i\, x_i^\top (X^\top W X)^{-1} x_i.
\]

\begin{proof}
The per-sample loss $\ell(\theta, z_j) = -y_j \log \hat{p}_j - (1 - y_j) \log(1 - \hat{p}_j)$ has gradient and Hessian
\[
\nabla_\theta \ell(\theta, z_j) = -(y_j - \hat{p}_j)\, x_j, \qquad \nabla_\theta^2 \ell(\theta, z_j) = w_j\, x_j x_j^\top,
\]
so the empirical Hessian is $\sum_{j=1}^n \nabla_\theta^2 \ell(\hat{\theta}, z_j) = X^\top W X$.

\paragraph{Influence.} Substituting in Eq.~\eqref{eq:inf}:
\[
\widehat{\Delta\ell}^{\mathrm{IF}}_i = (y_i - \hat{p}_i)^2\, x_i^\top (X^\top W X)^{-1} x_i = \frac{(y_i - \hat{p}_i)^2}{w_i}\, h_{ii}.
\]

\paragraph{Newton.} The leave-one-out Hessian is $X^\top W X - w_i x_i x_i^\top$. By Sherman-Morrison,
\[
x_i^\top \bigl(X^\top W X - w_i x_i x_i^\top\bigr)^{-1} x_i
= \frac{x_i^\top (X^\top W X)^{-1} x_i}{1 - w_i\, x_i^\top (X^\top W X)^{-1} x_i}
= \frac{h_{ii} / w_i}{1 - h_{ii}}.
\]
Substituting in Eq.~\eqref{eq:newton}:
\[
\widehat{\Delta\ell}^{\mathrm{NS}}_i = (y_i - \hat{p}_i)^2 \cdot \frac{h_{ii} / w_i}{1 - h_{ii}} = \frac{(y_i - \hat{p}_i)^2}{w_i} \cdot \frac{h_{ii}}{1 - h_{ii}}.
\]
\end{proof}

\subsection{\texorpdfstring{Proof that $\partial \hat{p}_i / \partial y_i = h_{ii}$}{Proof that dpi/dyi = hii}}
\begin{proof}
The MLE $\hat\theta$ satisfies the score equation
\begin{equation*}
    X^\top (y - \hat{p}) = 0,
\end{equation*}
where $\hat{p} = s(X\hat\theta)$. Differentiating with respect to $y_i$ and applying the chain rule, with $\partial \hat{p}/\partial \hat\eta = W$ and $\hat\eta = X\hat\theta$, yields
\begin{equation*}
    x_i \;=\; X^\top W X \,\frac{\partial \hat\theta}{\partial y_i},
    \qquad\text{hence}\qquad
    \frac{\partial \hat\theta}{\partial y_i} \;=\; (X^\top W X)^{-1} x_i.
\end{equation*}
Since $\partial \hat{p}_i / \partial y_i = w_i\, \partial \hat\eta_i / \partial y_i = w_i\, x_i^\top (\partial \hat\theta / \partial y_i)$, we obtain
\begin{equation*}
    \frac{\partial \hat{p}_i}{\partial y_i} \;=\; w_i\, x_i^\top (X^\top W X)^{-1} x_i \;=\; h_{ii}. \qedhere
\end{equation*}
\end{proof}

\subsection{Generalization to the Multivariate Case}
\label{app:sec4_multivariate}

We extend the closed-form expressions to multivariate-output settings: multivariate linear regression and multinomial logistic regression. The arguments rely only on the per-sample gradient and Hessian, so Eq.~\eqref{eq:inf} and Eq.~\eqref{eq:newton} apply. We derive the exact formulas for the influence and Newton estimates in these two cases.

\paragraph{Multivariate linear regression.} For multivariate linear regression over $q$ outputs with quadratic loss $\ell(\Theta) = \sum_{j=1}^n \|y_j - \Theta^\top x_j\|^2$, the problem decouples into $q$ \emph{independent} OLS regressions. For a single dimension $k$, the residual is $e_{i,k} \in \mathbb{R}$, the gradient is $g_{i,k} = -2 e_{i,k} x_i$, and the scalar leverage is $h_{ii} = x_i^\top (X^\top X)^{-1} x_i$.

\begin{proposition}[Influence Approximation]
The influence approximation (using the full-data Hessian) for the leave-one-out change in loss is:
\begin{equation}
\widehat{\Delta\ell}^{\mathrm{IF}}_i = 2 \|e_i\|^2 h_{ii}
\end{equation}
where $e_i = y_i - \hat{\Theta}^\top x_i \in \mathbb{R}^q$ is the full residual vector.
\end{proposition}

\begin{proof}
The full Hessian for dimension $k$ is $\mathcal{H}_\Theta = 2X^\top X$ and is independent of $k$. The loss gap for dimension $k$ using the influence approximation presented in equation \eqref{eq:inf} is:
\[
\widehat{\Delta\ell}^{\mathrm{IF}}_{ik} = g_{i,k}^\top \mathcal{H}_\Theta^{-1} g_{i,k} = (-2 e_{i,k} x_i)^\top (2X^\top X)^{-1} (-2 e_{i,k} x_i) = 2 e_{i,k}^2 x_i^\top(X^\top X)^{-1}x_i = 2 e_{i,k}^2 h_{ii}
\]
Summing the independent changes over all $q$ dimensions yields the total loss gap:
\[
\widehat{\Delta\ell}^{\mathrm{IF}}_i = \sum_{k=1}^q 2 e_{i,k}^2 h_{ii} = 2 h_{ii} \sum_{k=1}^q e_{i,k}^2 = 2 \|e_i\|^2 h_{ii}
\]
\end{proof}

\vspace{1em}

\begin{proposition}[Newton Approximation]
The Newton approximation, using the leave-one-out Hessian for the leave-one-out change in loss is:
\begin{equation}
\widehat{\Delta\ell}^{\mathrm{NS}}_i = 2 \|e_i\|^2 \frac{h_{ii}}{1 - h_{ii}}
\end{equation}
\end{proposition}

\begin{proof}
The exact LOO Hessian for dimension $k$ is $H_{-i} = 2(X^\top X - x_i x_i^\top)$. Applying the Sherman-Morrison formula to invert this rank-1 downdate yields $x_i^\top (X^\top X - x_i x_i^\top)^{-1} x_i = \frac{h_{ii}}{1 - h_{ii}}$. The loss gap per dimension using this LOO Hessian is:
\[
\widehat{\Delta\ell}^{\mathrm{NS}}_{ik} = g_{i,k}^\top H_{-i}^{-1} g_{i,k} = (-2 e_{i,k} x_i)^\top \left[ 2(X^\top X - x_i x_i^\top) \right]^{-1} (-2 e_{i,k} x_i)
\]
\[
\widehat{\Delta\ell}^{\mathrm{NS}}_{ik} = 4 e_{i,k}^2 \cdot \frac{1}{2} x_i^\top (X^\top X - x_i x_i^\top)^{-1} x_i = 2 e_{i,k}^2 \frac{h_{ii}}{1 - h_{ii}}
\]
Summing over all $q$ dimensions yields the total loss gap:
\[
\widehat{\Delta\ell}^{\mathrm{NS}}_i = \sum_{k=1}^q 2 e_{i,k}^2 \frac{h_{ii}}{1 - h_{ii}} = 2 \frac{h_{ii}}{1 - h_{ii}} \sum_{k=1}^q e_{i,k}^2 = 2 \|e_i\|^2 \frac{h_{ii}}{1 - h_{ii}}
\]
\end{proof}

\paragraph{Multivariate logistic regression.} We denote $y_j \in \{0,1\}^m$ the one-hot encoded label, $\hat{\Theta} \in \mathbb{R}^{p \times m}$ the learned parameter, and $\ell_j(\hat{\Theta}) = -\sum_{k=1}^m y_{jk} \log \hat{p}_{jk}$ the total loss, with $\hat{p}_j = \operatorname{softmax}(\hat{\Theta}^\top x_j)$ the predicted class probabilities. The results for the influence and Newton estimates are given in the following proposition.
\begin{proposition}
    In the multivariate case for multinomial logistic regression over $m$ classes, we have:
    \begin{equation}
        \widehat{\Delta\ell}^{\mathrm{NS}}_i = g_{i}^{\top}H_{ii}(I_{m}- W_{i}H_{ii})^{-1}g_{i}
    \end{equation}
    
    Where
    \begin{itemize}
        \item $g_{i} = \hat{p}_{i}-y_{i} \in \mathbb{R}^{m}$.
        \item $W_{i} =\operatorname{diag}(\hat{p}_{i})- \hat{p}_{i}\hat{p}_{i}^{\top} \in \mathbb{R}^{m\times m}$.
        \item $H_{ii} = \tilde{X}_{i}^{\top}\mathcal{H}_{\hat{\Theta}}^{-1}\tilde{X}_{i} \in \mathbb{R}^{m\times m}$ is the leverage matrix.
        \item $\tilde{X}_{i} = x_{i} \otimes I_{m} \in \mathbb{R}^{pm \times m}$ is the expanded feature matrix.
        \item $\mathcal{H}_{\hat{\Theta}} = \sum_{j=1}^{n} x_{j}x_{j}^{\top}\otimes W_{j} \in \mathbb{R}^{pm \times pm}$ is the full Hessian of the loss.
    \end{itemize}
\end{proposition}

\begin{proof}

Let $L(\hat{\Theta}) = \sum_{j=1}^n \ell_j(\hat{\Theta})$ be the total cross-entropy loss. For convenience, we parameterize the model by the vectorized weights $\hat{\Theta} \in \mathbb{R}^{pm}$ instead of a matrix $\in \mathbb{R}^{p \times m}$. Using the provided definitions, the per-sample gradient and Hessian evaluated at the optimal full-data parameters $\hat{\Theta}$ can be written compactly as:
\[\nabla_{\hat{\Theta}} \ell_i(\hat{\Theta}) = \tilde{X}_i g_i\]

\[\nabla_{\Theta}^2 \ell_i(\hat{\Theta}) = \tilde{X}_i W_i \tilde{X}_i^\top = (x_i \otimes I_m) W_i (x_i^\top \otimes I_m) = x_i x_i^\top \otimes W_i\]

Summing the per-sample Hessians gives the full Hessian $\mathcal{H}_{\hat{\Theta}} = \sum_{j=1}^{n} \tilde{X}_j W_j \tilde{X}_j^\top$. 
When we leave out the $i$-th sample, the leave-one-out Hessian $\mathcal{H}_{-i}$ is defined as:
\[
\mathcal{H}_{-i} = \mathcal{H}_{\hat{\Theta}} - \tilde{X}_i W_i \tilde{X}_i^\top
\]

Using a single Newton step approximation, the change in the estimated parameters upon removing the $i$-th sample is:
\[
\Delta \hat{\Theta}_i \approx \mathcal{H}_{-i}^{-1} \nabla_{\Theta} \ell_i(\hat{\Theta}) = \mathcal{H}_{-i}^{-1} \tilde{X}_i g_i
\]

The approximate change in the loss function for the left-out sample, $\Delta \ell_i$, evaluates the gradient against this parameter shift:
\[
\Delta \ell_i \approx \nabla_{\Theta} \ell_i(\hat{\Theta})^\top \Delta \hat{\Theta}_i = (\tilde{X}_i g_i)^\top \mathcal{H}_{-i}^{-1} (\tilde{X}_i g_i) = g_i^\top \left( \tilde{X}_i^\top \mathcal{H}_{-i}^{-1} \tilde{X}_i \right) g_i
\]

To evaluate the inner term without explicitly inverting the $pm \times pm$ LOO Hessian, we apply the Woodbury matrix identity to $\mathcal{H}_{-i}^{-1} = (\mathcal{H}_{\hat{\Theta}} - \tilde{X}_i W_i \tilde{X}_i^\top)^{-1}$. Because the weight matrix $W_i$ is not strictly invertible (its columns sum to zero), we use the form of the Woodbury identity that does not require $W_i^{-1}$:

\[(\mathcal{H}_{\hat{\Theta}} - \tilde{X}_i W_i \tilde{X}_i^\top)^{-1} = \mathcal{H}_{\hat{\Theta}}^{-1} + \mathcal{H}_{\hat{\Theta}}^{-1} \tilde{X}_i W_i (I_m - \tilde{X}_i^\top \mathcal{H}_{\hat{\Theta}}^{-1} \tilde{X}_i W_i)^{-1} \tilde{X}_i^\top \mathcal{H}_{\hat{\Theta}}^{-1}\]

We substitute this expanded inverse into our target expression $\tilde{X}_i^\top \mathcal{H}_{-i}^{-1} \tilde{X}_i$. Leveraging the definition $H_{ii} = \tilde{X}_i^\top \mathcal{H}_{\hat{\Theta}}^{-1} \tilde{X}_i$, we obtain:
\[\tilde{X}_i^\top \mathcal{H}_{-i}^{-1} \tilde{X}_i = H_{ii} + H_{ii} W_i (I_m - H_{ii} W_i)^{-1} H_{ii}\]

Next, we factor out $H_{ii}$ on the left side:
\[\tilde{X}_i^\top \mathcal{H}_{-i}^{-1} \tilde{X}_i = H_{ii} \left[ I_m + W_i (I_m - H_{ii} W_i)^{-1} H_{ii} \right]\]

Applying the push-through matrix identity $W_i(I_m - H_{ii}W_i)^{-1} = (I_m - W_i H_{ii})^{-1} W_i$, the term inside the brackets becomes:
\[\tilde{X}_i^\top \mathcal{H}_{-i}^{-1} \tilde{X}_i = H_{ii} \left[ I_m + (I_m - W_i H_{ii})^{-1} W_i H_{ii} \right]\]

To combine these terms, we multiply the identity matrix $I_m$ by $(I_m - W_i H_{ii})^{-1} (I_m - W_i H_{ii})$:
\[\tilde{X}_i^\top \mathcal{H}_{-i}^{-1} \tilde{X}_i = H_{ii} \left[ (I_m - W_i H_{ii})^{-1} (I_m - W_i H_{ii}) + (I_m - W_i H_{ii})^{-1} W_i H_{ii} \right]\]

Factoring out $(I_m - W_i H_{ii})^{-1}$ from both terms in the bracket leaves:
\[\tilde{X}_i^\top \mathcal{H}_{-i}^{-1} \tilde{X}_i = H_{ii} (I_m - W_i H_{ii})^{-1} \left( I_m - W_i H_{ii} + W_i H_{ii} \right)\]

The $- W_i H_{ii}$ and $+ W_i H_{ii}$ cleanly cancel out, yielding the final simplified matrix expression:
\[\tilde{X}_i^\top \mathcal{H}_{-i}^{-1} \tilde{X}_i = H_{ii} (I_m - W_i H_{ii})^{-1}\]

Substituting this matrix back into our earlier equation for the scalar change in loss yields the desired formulation:
\[
\widehat{\Delta\ell}^{\mathrm{NS}}_i = g_{i}^{\top} H_{ii} (I_{m}- W_{i}H_{ii})^{-1} g_{i}
\]
\end{proof}
\begin{proposition}
    In the multivariate case for multinomial logistic regression over $m$ classes, we have:
    \begin{equation}
        \widehat{\Delta\ell}^{\mathrm{IF}}_i = g_{i}^{\top}H_{ii}g_{i}
    \end{equation}  
    Where $g_i$, $H_{ii}$ are defined as in the previous proposition.
\end{proposition}

\begin{proof}
    The influence function approximation is a special case of the Newton step approximation. This corresponds to replacing the matrix inverse $(I_m - W_i H_{ii})^{-1}$ with the identity matrix $I_m$. 
\end{proof}

For $m = 2$ the per-sample weight collapses to the scalar $w_j$, $\mathcal{H}_{\hat{\Theta}} = X^\top W X$ factorizes as a single Kronecker product, and we recover the binary formulas of the previous subsection. In general $\mathcal{H}$ does \emph{not} factorize as a Kronecker product because $W_j$ depends on $j$, so the multinomial leverage cannot be separated into a sample term and a class term as in the binary case.

\section{Details on the Experimental Setup}
\label{app:exp_details}

\subsection{Model Architectures and Training Hyperparameters}
\label{app:arch_hyperparams}

We evaluate our surrogates on a deliberately diverse benchmark spanning tabular regression, image classification, and text classification, with model sizes ranging from a single linear layer to a fully fine-tuned transformer. For each (dataset, architecture) pair, target and reference (shadow) models are trained with the \emph{same} hyperparameters; only the random seed, which controls both weight initialization and the training split differs. We train 200 reference models and 16 target models per setting. Table~\ref{tab:archs} summarizes the architectures and Table~\ref{tab:hparams} the training hyperparameters. The detailed implementation can be found in the codebase.

\begin{table}[htbp]
    \centering
    \small
    \caption{Datasets and architectures used in our benchmark. $|\mathcal{D}|$ denotes the number of training examples used to draw membership splits. \#classes is only relevant for classification tasks.}
    \setlength{\tabcolsep}{4pt}
    \begin{tabular}{llllrr}
        \toprule
        Dataset & Task & Model & Backbone & $|\mathcal{D}|$ & \#classes \\
        \midrule
        UCI Energy            & Reg.     & MLP      & 2 hidden blocks $[64,32]$, dropout $0.3$ & 614     & --- \\
        California Housing    & Reg.     & Ridge Reg.       & 1 linear layer             & 16{,}512 & --- \\
        California Housing    & Reg.     & MLP             & 3 hidden blocks $[128, 128, 128]$           & 16{,}512 & --- \\
        BloodMNIST            & Classif. & DINO   & DINOv2 ViT-S/14 (frozen + linear head)          & 8{,}000  & 8 \\ 
        CIFAR-10              & Classif. & CNN             & 3 conv blocks       & 50{,}000 & 10 \\
        CIFAR-100             & Classif. & ResNet-20       & 3 stages, dropout $0.3$ & 50{,}000 & 100 \\
        SST-2                 & Classif. & BERT & BERT-tiny (L=2, H=128)       & 67{,}349 & 2 \\
        SST-2                 & Classif. & BERT  & BERT-base (frozen + linear head)   & 67{,}349 & 2 \\
        \bottomrule
    \end{tabular}
    \label{tab:archs}
\end{table}

\begin{table}[htbp]
\centering
\small
\caption{Training hyperparameters. Loss is quadratic for regression and cross-entropy for classification.}
\setlength{\tabcolsep}{4pt}
\begin{tabular}{llrrrlr}
\toprule
Dataset & Model & Epochs & Batch & LR & Optimizer & Weight decay \\
\midrule
UCI Energy           & MLP     & 300 & 32   & $10^{-3}$ & Adam          & $10^{-3}$  \\
California Housing   & Ridge Reg.     & 200 & 256  & $10^{-2}$ & Adam          & $10^{-3}$  \\
California Housing   & MLP            & 200 & 256  & $10^{-3}$ & Adam          & $5\!\cdot\!10^{-4}$  \\
BloodMNIST           & DINO (linear)  & 100 & 512  & $10^{-2}$ & SGD ($\mu\!=\!0.9$) & $10^{-6}$  \\
CIFAR-10             & CNN            & 25 & 256  & $10^{-3}$ & AdamW & $5\!\cdot\!10^{-4}$ \\
CIFAR-100            & ResNet-20      & 50 & 256  & $10^{-3}$ & SGD ($\mu\!=\!0.9$) & $5\!\cdot\!10^{-4}$ \\
SST-2                & BERT (full FT) & 10  & 32   & $10^{-4}$ & AdamW         & $10^{-2}$  \\
SST-2                & BERT (linear)  & 10 & 32   & $10^{-3}$ & AdamW         & $10^{-2}$ \\
\bottomrule
\end{tabular}
\label{tab:hparams}
\end{table}

\subsection{Time and Compute Resources}
\label{app:compute_time}

All experiments were run using V100 (32\,GB) GPU. The values presented in Table~\ref{tab:compute} are the wall-clock times reported using dedicated packages. They split into four phases: (i) training the reference ensemble, (ii) training the target models, (iii) running the LiRA, RMIA and Shokri attacks plus the per-point statistics, and (iv) computing the leverage-score / influence / Newton surrogates on each target. We report the total time spent in each phase, as well as the total time for the entire benchmark. 

To evaluate cost-effectiveness, surrogate performance is reported as the time required to compute scores for a single target, allowing a direct comparison with the training of the $200$ reference models. While actual implementation involves computing scores across all targets, resulting in a total runtime approximately 16 times the per-target figure, the single-target metric highlights the efficiency gains. For example, using the DINO linear model, computing surrogates for one target takes only 15 seconds; this represents a $50{,}000 \times$ speedup compared to the $60$ hours required to train all the reference (shadow) models.

\begin{table}[t]
    \centering
    \caption{Compute budget for our experimental setting. Values are observed run-times in GPU-hours. \emph{Surrogates} aggregates the GPU time spent computing the influence / Newton scores surrogates over all targets.}
\footnotesize
\setlength{\tabcolsep}{3pt}
\begin{tabular}{llrrrrrr}
\toprule
&  \multicolumn{3}{c}{Training}  & Attacks+Stats & Surrogates & IF (full network) & Total \\
\cmidrule(l){2-4}\cmidrule(lr){5-5}\cmidrule(lr){6-6}\cmidrule(lr){7-7}\cmidrule(lr){8-8}
Dataset / Model  &  per-model & 200 refs  & 16 targets   & 200 refs& per-target & per-target &GPU-h \\
\midrule
UCI Energy / MLP            & 30\,s   & 1.7\,h  & 8\,min  & 30\,s   & 3\,s   & -- & 1.7\,h  \\
California / Ridge reg.    & 1\,min  & 3.3\,h  & 0.3\,h  & 7\,min  & 10\,s  & -- & 3.6\,h  \\
California / MLP            & 1\,min  & 3.3\,h  & 0.3\,h  & 7\,min  & 10\,s  & 20\,min & 8.9\,h  \\
BloodMNIST / DINO (linear)  & 15\,min & 50.0\,h & 4.0\,h  & 5\,min  & 15\,s  & --& 54.0\,h \\
CIFAR-10 / CNN              & 4\,min  & 13.0\,h & 1.0\,h  & 20\,min & 30\,s  & 2.2\,h& 49.2\,h \\
CIFAR-100 / ResNet-20       & 13\,min & 44.4\,h & 3.5\,h  & 45\,min & 1\,min &-- & 48.5\,h \\
SST-2 / BERT-tiny (full FT) & 6\,min  & 20\,h   & 1.6\,h  & 17\,min & 45\,s  &--& 20.6\,h \\
SST-2 / BERT-base (linear)  & 19\,min &64.6\,h  & 5.0\,h  & 18\,min & 3\,min  &-- & 70\,h  \\
ImageNet-100/ ResNet-50  & 43\,min &143.3\,h  & 11.2\,h  & 36\,min & 6\,min  &-- & 191.8\,h  \\
\addlinespace
\textbf{Total}  &    &   &      &      &       & &\textbf{447 GPU-h} \\
\bottomrule
\end{tabular}
\label{tab:compute}
\end{table}

The training phase dominates the total cost. Once all models are trained, surrogates and attacks reuse the same models.

\section{Additional Experimental Results}
\label{app:add_exp_results}
Detailed results for the LiRA ASR top-$k$\% recall are reported in Table~\ref{tab:lira_summary}. We observe that the Newton and Influence surrogates outperform the loss surrogate across most datasets and architectures. Although the Spearman correlation is often higher for our surrogates, the absolute values remain quite low (except for the CIFAR datasets).

This  suggests that the ranking of samples by our surrogates is not perfectly aligned with the ranking by the LiRA score, mostly for the points that are in the middle of the distribution. This is not surprising as the LiRA ASR score is very noisy in this part of the distribution. This intuition is confirmed by Figure~\ref{fig:scatter_loss_influence_lira}: points in the middle of the distribution are scattered, whereas in the tails the ranking induced by the surrogates is much more consistent with the LiRA score.

\subsection{Comparison with Full-Network Influence functions}

To evaluate the impact of the last-layer approximation, we computed full-network influence functions for CIFAR-10 (CNN) and California Housing (MLP). These estimates were calculated using inverse Hessian-vector products \citep{hessianProduct} via the Conjugate Gradient (CG) algorithm \citep{hestenesMethodsConjugateGradients1952}. We ensured CG convergence by setting a tolerance of $10^{-2}$ and applying a notably high damping value of $10$; while strictly necessary for the algorithm to converge in our case, such heavy damping inherently alters the resulting influence estimates.

As presented in Table~\ref{tab:compute}, the computational overhead for the full-network approach is substantial: for CIFAR-10, processing a single model took more than 2 hours, compared to only 30 seconds for the last-layer surrogate. Despite this significantly higher cost, the full-network influence function performed worse in identifying vulnerable samples. This suggests that the last-layer approximation not only offers a massive gain in efficiency but may also act as a useful regularizer by focusing on the most discriminative features. Furthermore, as noted by \citet{basu2021influence}, this underperformance may come from the inherent fragility and poor scaling of influence functions when applied to the non-convex landscapes of deep networks.

Figure~\ref{fig:full_inf} illustrates the top-$k$\% recall with the additional full network influence function surrogate $\nabla \ell_i^\top H_\theta^{-1}\nabla \ell_i$. The values are reported in the CIFAR-10 and California Housing (MLP) cells in Table~\ref{tab:lira_summary}.

\label{app:full_inf}

\begin{figure}[htbp]
    \center
    \includegraphics{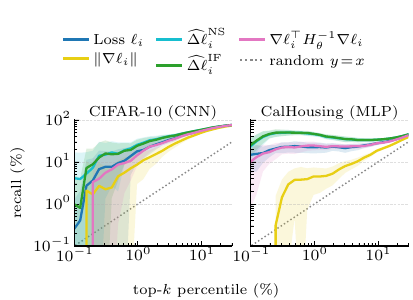}
    \caption{Recall of the $k$\% most vulnerable to LiRA (ranked by the per-sample Attack Success Rate) when keeping the top-$k$\% for each of the classical surrogates as well as the full Influence Function surrogate $\nabla \ell_i^\top H_\theta^{-1}\nabla \ell_i$. Axes are log-log; the dotted diagonal $y=x$ is the random baseline. Solid lines are means over $16$ target models with $\pm 1$ std bands.}
    \label{fig:full_inf}
\end{figure}

\subsection{Per-Sample TPR@low-FPR as Baseline}
\label{app:tpr_low_fpr}
For completeness, we also evaluated the performance of all the proxies when using the per-sample TPR@low-FPR as the ground truth. We fixed each threshold to obtain a FPR of 1\% and average the results over 16 target models.

Table~\ref{tab:tpr_at_fpr} presents the Recall (\%) of the ground-truth top 1\% (ranked by per-sample TPR@FPR=1\%) when selecting the top 5\% according to each surrogate. Our surrogates still correlate better with this metric than the loss baseline.

\begin{table}[hp]
\centering
\small
\setlength{\tabcolsep}{4pt}
\caption{Recall of the ground-truth top 1\% most vulnerable samples (ranked by LiRA \textbf{per-sample TPR@FPR=1\%}) when selecting the top 5\% according to each surrogate.  Mean $\pm$ std across target models. Higher is better; \textbf{bold} = best per dataset.}
\label{tab:tpr_at_fpr}
\begin{tabular}{llccc}
\toprule
& & \multicolumn{3}{c}{Recall of Top 1\% @5\% (in \%)} \\
\cmidrule(lr){3-5}
Dataset & Model & Loss $\ell_i$ & $\widehat{\Delta\ell}^{\mathrm{NS}}_i$ & $\widehat{\Delta\ell}^{\mathrm{IF}}_i$ \\
\midrule
SST-2  & BERT, full ft. &	37.1 $\pm$ 2.6 \% 	&37.1 $\pm$ 2.9 \% 	&\textbf{37.2 $\pm$ 2.9 \%} \\
CIFAR-10  &CNN 	& 58.2 $\pm$ 1.7 \% 	&\textbf{67.0 $\pm$ 1.7 \%} &	66.4 $\pm$ 1.7 \% \\
CIFAR-100 &ResNet-20 &	26.7 $\pm$ 2.5 \% 	& \textbf{37.5 $\pm$ 5.3 \%} &	37.1 $\pm$ 5.1 \% \\
BloodMNIST &DINOv2	& 78.6 $\pm$ 4.8 \% 	&86.0 $\pm$ 11.2 \% &	\textbf{86.5 $\pm$ 6.6 \%} \\

\bottomrule
\end{tabular}
\end{table}

\subsection{Evaluation on Non-Members}
\label{app:non_members}
Throughout our work, we framed all our surrogates from a member's point of view: how to estimate the loss gap $\ell_{-i} - \ell_{i}$, when the loss $\ell_{i}$ is available. We made this choice, as we believe it is the primary setting of interest. A model provider or auditor wants to assess which points in their datasets are vulnerable to costly MIAs.

One could instead ask to estimate the vulnerability of a point that is \emph{not} in the training set, for instance before adding it. This corresponds to estimating the loss gap when only $\ell_{-i}$ is available. Since Proposition~\ref{prop:gap} gives a property of the data point itself, independent of its membership status, it can equally be used to derive an estimator of the loss gap for non-members, requiring a rank-one update rather than a downdate. For example, for linear regression, this yields:

\[\widehat{\Delta \ell_{i}}^{NS, +} = \frac{2e_{-i}^2 h_{-ii}}{1+h_{-ii}}\] with $h_{-ii} = x_i^\top G_{-i}^{-1}x_i$. There is an analogous correction for logistic regression. 

To demonstrate this in practice, we computed our surrogates across several datasets for non-members. For each specific target model, all measures were computed exclusively on points held out from its training set. The results in Table~\ref{tab:non_members} confirm that our estimators also work in identifying the most vulnerable non-members.
\begin{table}[htbp]
    \centering
    \small
    \caption{Recall of the ground-truth top 1\% most vulnerable samples (ranked by  LiRA per-sample ASR) when selecting the top 5\% according to each surrogate for points that were \textbf{out} of the training set.  Mean $\pm$ std across target models. Higher is better; \textbf{bold} = best per dataset.}
    \label{tab:non_members}
    \begin{tabular}{llcccc}
    \toprule
        & & \multicolumn{4}{c}{Recall of Top 1\% @5\% (in \%)} \\
    \cmidrule(lr){3-6}
        Dataset & Model & Loss $\ell_{-i}$  & $\|\nabla \ell_{-i}\|$ & Newton out & Influence out \\
        \midrule
        CalHousing & Ridge regression & \textbf{12.1 $\pm$ 2.5} & \textbf{12.1 $\pm$ 2.5} & 11.3 $\pm$ 2.6 & 11.3 $\pm$ 2.6 \\
        CalHousing & MLP & 57.1 $\pm$ 9.4 & 22.7 $\pm$ 7.3 & 79.8 $\pm$ 5.5 & \textbf{80.3 $\pm$ 5.2} \\
        CIFAR-10 & CNN & 74.6 $\pm$ 2.5 & 30.8 $\pm$ 2.1 & \textbf{78.8 $\pm$ 2.6} & 78.7 $\pm$ 2.5 \\
        CIFAR-100 & ResNet-20 & 43.2 $\pm$ 2.5 & 10.1 $\pm$ 1.9 & \textbf{47.8 $\pm$ 6.2} & 46.9 $\pm$ 6.1 \\
        BloodMNIST & DINOv2  & 85.7 $\pm$ 4.3 & 88.1 $\pm$ 3.1 & 95.9 $\pm$ 2.5 & \textbf{98.1 $\pm$ 1.2} \\
        UCI Energy & tab. MLP & 25.0 $\pm$ 33.3 & 18.8 $\pm$ 32.1 & \textbf{54.2 $\pm$ 31.9} & \textbf{54.2 $\pm$ 31.9}\\
        \bottomrule
    \end{tabular}
\end{table}

\subsection{Evaluation on a Larger Dataset}
\label{app:resnet}
We additionally evaluated a ResNet-50~\citep{resnet} classification model trained on ImageNet-100~\citep{dengImagenet} (restriction of ImageNet to 100 classes). Due to the large size of the last layer, we used a K-FAC approximation~\citep{martensKfac} to compute the scores. We trained 200 shadow models and 16 target models using a standard 45-epoch schedule. At this rate, depending on the seed, some models interpolated, yielding up to 99\% training accuracy. We present in Table~\ref{tab:imagenet} the results averaged over 16 target models, with the LiRA ASR as target metric.

\begin{table}[htbp]
    \centering
    \small
    \caption{Recall of the ground-truth top 1\% most vulnerable samples (ranked by  LiRA per-sample ASR) when selecting the top 5\% according to each surrogate for the \textbf{ImageNet} dataset. Mean $\pm$ std across target models. Higher is better; \textbf{bold} = best per dataset.}
    \label{tab:imagenet}
    \begin{tabular}{llcccc}
    \toprule
    Dataset & Model & Loss $\ell_i$ & $\|\nabla \ell_{i}\|$ &$\widehat{\Delta\ell}^{\mathrm{NS}}_i$ & $\widehat{\Delta\ell}^{\mathrm{IF}}_i$ \\
        \midrule
        ImageNet-100 & ResNet-50 &  \textbf{22.14 $\pm$ 5.38 \%}& 	18.78 $\pm$ 4.35 \% &	18.31 $\pm$ 3.80 \% 	&18.63 $\pm$ 3.79 \% \\
        \bottomrule
    \end{tabular}
\end{table}

In this specific case, our estimates performed worse than the loss baseline. We believe this discrepancy stems from three possible factors:
\begin{itemize}
    \item The interpolation regime, which requires a different analysis since many samples are highly vulnerable. ImageNet-100 is by far the most vulnerable dataset in our suite (mean per-sample LiRA ASR of 0.83, compared to 0.77, 0.68, and 0.57 for CIFAR-100, CIFAR-10, and BloodMNIST, respectively). As a model approaches interpolation, the residual training loss ranks the margin rather than the learning difficulty, and all loss-derived proxies degrade.
    \item The leverage score carries less information in this interpolating regime. It has a Spearman rank correlation with the loss of 0.91 (compared to 0.59 for CIFAR-100), meaning they contain the same information in terms of vulnerability.
    \item The use of the K-FAC approximation can also alter the results.
\end{itemize}

\subsection{\texorpdfstring{Top-$k$ Recall and Spearman Correlation for LiRA Attack}{Top-k Recall and Spearman Correlation for Lira Attack}}
    \label{app:lira_full_results}
    
We report full results for the LiRA MIA in Table~\ref{tab:lira_summary}. Notably, we report the Spearman rank correlation between the surrogate scores and the per-point ASR for reference, but we do not find a strong correlation. We attribute this to the fact that points near the center of the vulnerability distribution yield an Attack Success Rate (ASR) close to $0.5$, offering little informative value. In contrast, top-$k$ recall effectively isolates the highly vulnerable points, which are the primary focus for practical security auditing and defense. 

For classification models we also report the self-entropy defined as $\mathrm{H}_i = -\sum_{k=1}^{m} \hat{p}_{i,k} \log \hat{p}_{i,k}$, where $\hat{p}_{i,k}$ is the model's predicted probability of class $k$ on input $x_i$. Self-entropy is a standard confidence-based MIA baseline~\citep{salem2018ml, song2021systematic}: training samples typically receive more peaked predictions and hence lower entropy than non-members.

\begin{table}[t]
\centering
\small
\renewcommand{\arraystretch}{1.2}
\setlength{\tabcolsep}{4pt}
\caption{LiRA membership-inference attack: top-$k$ recall (\%) and Spearman rank correlation. Best results per dataset are in \textbf{bold}. In contrast to Table~\ref{tab:proxy-recall-at-5pct-lira}, the recall values here are computed by keeping the top-$k$ samples with the highest surrogate scores and checking how many of them are in the top-$k$ samples with the highest LiRA ASR score.}
\label{tab:lira_summary}
\begin{tabular}{ll ccc c}
\toprule

& & \multicolumn{3}{c}{Recall (\%)} & \\
\cmidrule(lr){3-5}
Dataset & Surrogate & @0.1\% & @1\% & @5\% & Spearman $\rho$ \\
\midrule
\multirow{5}{*}{\makecell[l]{SST-2 \\ (BERT, full ft.)}} & Loss $\ell_i$ & $0.00 \pm 0.00$ & $5.44 \pm 1.36$ & $30.83 \pm 1.11$ & $\mathbf{0.474} \pm \mathbf{0.022}$ \\
 & $\|\nabla\ell_i\|$ & $0.00 \pm 0.00$ & $0.57 \pm 0.18$ & $4.64 \pm 0.33$ & -- \\
 & Self-entropy & $0.02 \pm 0.06$ & $6.34 \pm 1.16$ & $28.52 \pm 1.10$ & $0.470 \pm 0.022$ \\
 & $\widehat{\Delta\ell}^{\mathrm{NS}}_i$ & $1.02 \pm 1.47$ & $\mathbf{7.05} \pm \mathbf{1.23}$ & $\mathbf{31.30} \pm \mathbf{1.44}$ & $0.389 \pm 0.036$ \\
 & $\widehat{\Delta\ell}^{\mathrm{IF}}_i$ & $\mathbf{1.22} \pm \mathbf{1.54}$ & $6.88 \pm 1.30$ & $31.28 \pm 1.44$ & $0.389 \pm 0.036$ \\
\midrule
\multirow{5}{*}{\makecell[l]{SST-2 \\ (BERT, linear probe)}} & Loss $\ell_i$ & $0.05 \pm 0.08$ & $\mathbf{2.30} \pm \mathbf{0.64}$ & $\mathbf{7.94} \pm \mathbf{0.50}$ & $\mathbf{0.045} \pm \mathbf{0.005}$ \\
 & $\|\nabla\ell_i\|$ & $0.00 \pm 0.00$ & $1.56 \pm 0.70$ & $6.64 \pm 0.65$ & -- \\
 & Self-entropy & $\mathbf{0.60} \pm \mathbf{1.26}$ & $1.07 \pm 0.46$ & $5.60 \pm 0.34$ & $0.026 \pm 0.005$ \\
 & $\widehat{\Delta\ell}^{\mathrm{NS}}_i$ & $0.00 \pm 0.00$ & $1.80 \pm 0.87$ & $6.88 \pm 0.57$ & $0.037 \pm 0.005$ \\
 & $\widehat{\Delta\ell}^{\mathrm{IF}}_i$ & $0.30 \pm 0.94$ & $2.27 \pm 0.70$ & $7.01 \pm 0.43$ & $0.038 \pm 0.005$ \\
\midrule
\multirow{4}{*}{\makecell[l]{CalHousing \\ (ridge)}} & Loss $\ell_i$ & $6.33 \pm 7.74$ & $3.16 \pm 1.81$ & $8.65 \pm 1.05$ & $0.033 \pm 0.006$ \\
 & $\|\nabla\ell_i\|$ & $6.33 \pm 7.74$ & $3.16 \pm 1.81$ & $8.65 \pm 1.05$ & $0.033 \pm 0.006$ \\
 & $\widehat{\Delta\ell}^{\mathrm{NS}}_i$ & $36.44 \pm 15.34$ & $\mathbf{7.54} \pm \mathbf{2.50}$ & $9.30 \pm 1.23$ & $0.050 \pm 0.007$ \\
 & $\widehat{\Delta\ell}^{\mathrm{IF}}_i$ & $\mathbf{38.02} \pm \mathbf{14.64}$ & $\mathbf{7.54} \pm \mathbf{2.50}$ & $\mathbf{9.30} \pm \mathbf{1.24}$ & $\mathbf{0.050} \pm \mathbf{0.007}$ \\
\midrule
\multirow{4}{*}{\makecell[l]{CalHousing \\ (MLP)}} & Loss $\ell_i$ & $14.90 \pm 9.25$ & $22.59 \pm 6.49$ & $23.58 \pm 4.03$ & $0.276 \pm 0.028$ \\
 & $\|\nabla\ell_i\|$ & $0.00 \pm 0.00$ & $4.53 \pm 2.83$ & $10.81 \pm 2.26$ & $0.198 \pm 0.019$ \\
 & $\widehat{\Delta\ell}^{\mathrm{NS}}_i$ & $23.95 \pm 11.89$ & $45.84 \pm 5.54$ & $33.22 \pm 3.96$ & $0.297 \pm 0.027$ \\
 & $\widehat{\Delta\ell}^{\mathrm{IF}}_i$ & $\mathbf{25.49} \pm \mathbf{11.97}$ & $\mathbf{46.46} \pm \mathbf{4.79}$ & $\mathbf{33.46} \pm \mathbf{3.88}$ & $\mathbf{0.297} \pm \mathbf{0.027}$ \\
 & Full influence & $10.35 \pm 9.35$ & $24.57 \pm 5.55$ & $23.72 \pm 2.85$ & $0.266 \pm 0.025$ \\

\midrule
\multirow{5}{*}{\makecell[l]{CIFAR-10 \\ (CNN)}} & Loss $\ell_i$ & $4.02 \pm 3.23$ & $17.30 \pm 2.10$ & $42.98 \pm 1.17$ & $0.813 \pm 0.004$ \\
 & $\|\nabla\ell_i\|$ & $0.01 \pm 0.05$ & $8.40 \pm 1.65$ & $33.52 \pm 1.28$ & $0.811 \pm 0.003$ \\
 & Self-entropy & $0.78 \pm 1.57$ & $10.36 \pm 1.33$ & $38.70 \pm 1.35$ & $0.808 \pm 0.004$ \\
 & $\widehat{\Delta\ell}^{\mathrm{NS}}_i$ & $\mathbf{15.28} \pm \mathbf{6.95}$ & $\mathbf{25.59} \pm \mathbf{1.77}$ & $\mathbf{48.14} \pm \mathbf{1.27}$ & $0.814 \pm 0.003$ \\
 & $\widehat{\Delta\ell}^{\mathrm{IF}}_i$ & $14.07 \pm 7.06$ & $24.48 \pm 1.99$ & $47.76 \pm 1.31$ & $0.814 \pm 0.003$ \\
 & Full influence & $0.00 \pm 0.00$ & $14.94 \pm 6.07$ & $41.49 \pm 3.19$ & $\mathbf{0.821} \pm \mathbf{0.009}$ \\

\midrule
\multirow{5}{*}{\makecell[l]{CIFAR-100 \\ (ResNet-20)}} & Loss $\ell_i$ & $2.26 \pm 2.86$ & $7.40 \pm 1.80$ & $20.53 \pm 0.72$ & $0.658 \pm 0.007$ \\
 & $\|\nabla\ell_i\|$ & $0.51 \pm 1.39$ & $2.68 \pm 0.69$ & $13.94 \pm 0.93$ & $0.659 \pm 0.005$ \\
 & Self-entropy & $0.25 \pm 0.99$ & $3.84 \pm 1.15$ & $19.75 \pm 1.11$ & $0.686 \pm 0.007$ \\
 & $\widehat{\Delta\ell}^{\mathrm{NS}}_i$ & $4.84 \pm 4.60$ & $\mathbf{13.51} \pm \mathbf{2.82}$ & $\mathbf{28.64} \pm \mathbf{3.15}$ & $\mathbf{0.690} \pm \mathbf{0.024}$ \\
 & $\widehat{\Delta\ell}^{\mathrm{IF}}_i$ & $\mathbf{5.77} \pm \mathbf{3.82}$ & $13.08 \pm 2.66$ & $28.24 \pm 2.80$ & $0.689 \pm 0.022$ \\
\midrule
\multirow{5}{*}{\makecell[l]{BloodMNIST \\ (DINOv2)}} & Loss $\ell_i$ & $10.55 \pm 10.08$ & $33.22 \pm 3.72$ & $59.06 \pm 2.19$ & $0.600 \pm 0.005$ \\
 & $\|\nabla\ell_i\|$ & $10.91 \pm 14.54$ & $39.95 \pm 4.54$ & $62.34 \pm 2.07$ & $0.604 \pm 0.005$ \\
 & Self-entropy & $0.00 \pm 0.00$ & $14.67 \pm 3.72$ & $44.45 \pm 2.05$ & $0.586 \pm 0.005$ \\
 & $\widehat{\Delta\ell}^{\mathrm{NS}}_i$ & $\mathbf{34.70} \pm \mathbf{26.66}$ & $\mathbf{52.68} \pm \mathbf{15.24}$ & $67.00 \pm 10.01$ & $0.604 \pm 0.031$ \\
 & $\widehat{\Delta\ell}^{\mathrm{IF}}_i$ & $27.24 \pm 17.77$ & $47.27 \pm 7.87$ & $\mathbf{67.09} \pm \mathbf{6.55}$ & $\mathbf{0.608} \pm \mathbf{0.024}$ \\
\midrule
\multirow{4}{*}{\makecell[l]{UCI Energy \\ (tab.\ MLP)}} & Loss $\ell_i$ & $0.00 \pm 0.00$ & $2.01 \pm 8.06$ & $20.74 \pm 4.87$ & $0.528 \pm 0.040$ \\
 & $\|\nabla\ell_i\|$ & $0.00 \pm 0.00$ & $6.25 \pm 12.92$ & $26.83 \pm 9.21$ & $0.507 \pm 0.036$ \\
 & $\widehat{\Delta\ell}^{\mathrm{NS}}_i$ & $0.00 \pm 0.00$ & $10.90 \pm 19.74$ & $28.25 \pm 7.50$ & $\mathbf{0.541} \pm \mathbf{0.032}$ \\
 & $\widehat{\Delta\ell}^{\mathrm{IF}}_i$ & $\mathbf{6.25} \pm \mathbf{25.00}$ & $\mathbf{14.72} \pm \mathbf{23.99}$ & $\mathbf{30.32} \pm \mathbf{7.89}$ & $0.540 \pm 0.032$ \\
\bottomrule
\end{tabular}
\end{table}

\subsection{\texorpdfstring{Top-$k$ Recall and Spearman Correlation for RMIA Attack and Classifier Shokri's Attack}{Top-k Recall and Spearman Correlation for RMIA Attack and Classifier Shokri's Attack}}

\label{app:rmia_shokri_full_results}

The RMIA~\citep{zarifzadehLowCostHighPowerMembership2024} attack is computed in the online setting, while the classifier-based~\citep{shokriMembershipInferenceAttacks2017} approach uses a Random Forest~\citep{breimanRandomForests2001} classifier trained with the logits of the target model as input. We provide the top-$k$ recall curves for RMIA and Shokri's attack in Figure~\ref{fig:topk_rmia} and Figure~\ref{fig:topk_shokri} respectively.

\begin{figure}[htbp]
  \centering
    \includegraphics{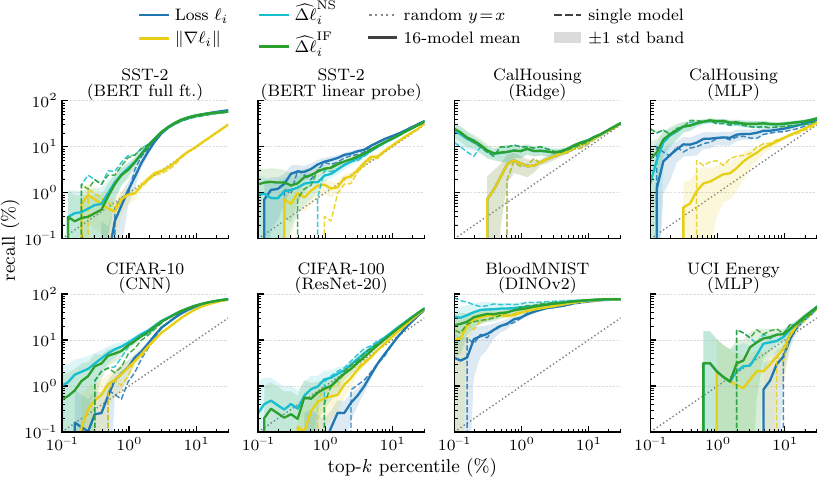} 
    \caption{Recall of the $k$\% most vulnerable to \textbf{RMIA} (ranked by the per-sample Attack Success Rate) when keeping the top-$k$\% for each of the four surrogates: the loss $\ell_i$, the gradient norm $\|\nabla\ell_i\|$, the influence estimate $\widehat{\Delta\ell}^{\mathrm{IF}}_i$, and the Newton estimate $\widehat{\Delta\ell}^{\mathrm{NS}}_i$. Axes are log-log; the dotted diagonal $y=x$ is the random baseline. Solid lines are means over $16$ target models with $\pm 1$ std bands; dashed-lines show a single representative model.}
    \label{fig:topk_rmia}
\end{figure}

\begin{figure}[htbp]
  \centering
    \includegraphics{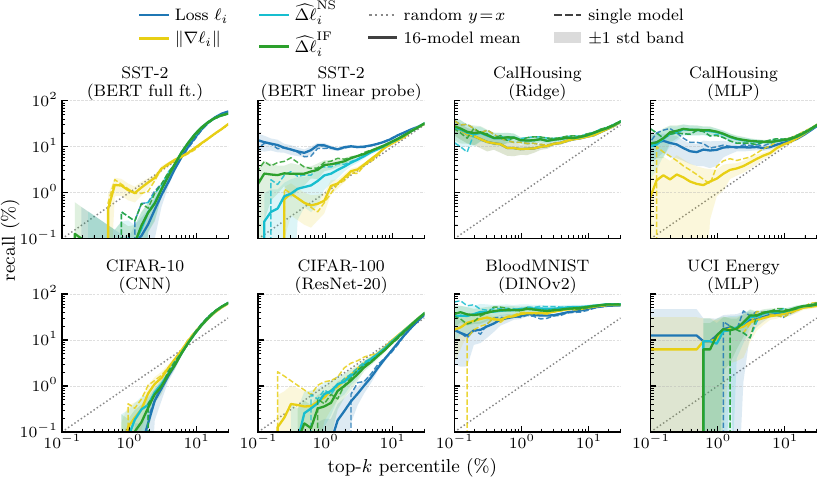} 
    \caption{Recall of the $k$\% most vulnerable to Shokri's attack (ranked by the per-sample Attack Success Rate) when keeping the top-$k$\% for each of the four surrogates: the loss $\ell_i$, the gradient norm $\|\nabla\ell_i\|$, the influence estimate $\widehat{\Delta\ell}^{\mathrm{IF}}_i$, and the Newton estimate $\widehat{\Delta\ell}^{\mathrm{NS}}_i$. Axes are log-log; the dotted diagonal $y=x$ is the random baseline. Solid lines are means over $16$ target models with $\pm 1$ std bands; dashed-lines show a single representative model.}
    \label{fig:topk_shokri}
\end{figure}

\subsection{Additional Visualizations}
We provide additional visualizations of the samples with the highest and lowest influence scores and Newton scores for the BloodMNIST, CIFAR-10 and CIFAR-100 datasets in Figure~\ref{fig:top_samples_images}. We also provide the top-10 highest and lowest influence scores and Newton scores for the BERT model trained on the SST-2 dataset with linear probing in Table~\ref{tab:bert_top_sents_bert_freeze}. The joint distribution of local loss $\ell_i$ and influence score for every training sample of a single target model, colored by per-sample LiRA ASR, is presented in Figure~\ref{fig:scatter_loss_influence_lira}.

\begin{figure}[htbp]
  \centering
    \includegraphics{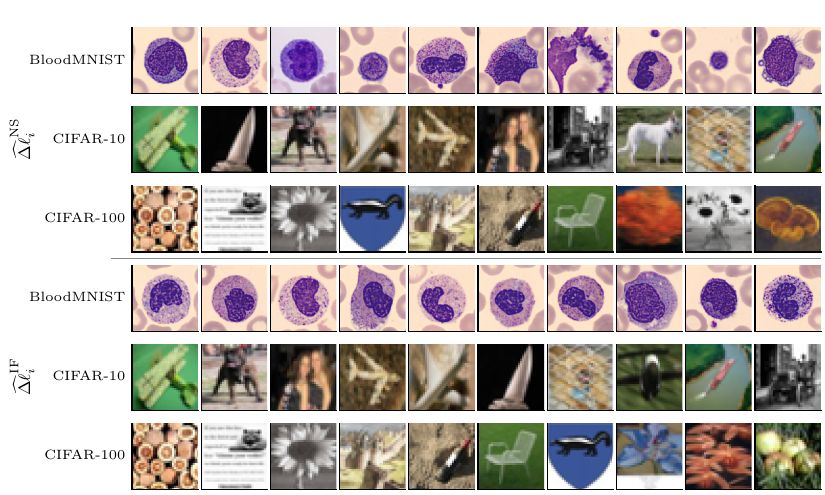}
    \caption{Visualization of the 10 images with the highest influence and Newton scores for BloodMNIST (top), CIFAR-10 (middle) and CIFAR-100 (bottom).}
    \label{fig:top_samples_images}
\end{figure}

\begin{figure}[htbp]
  \centering
    \includegraphics{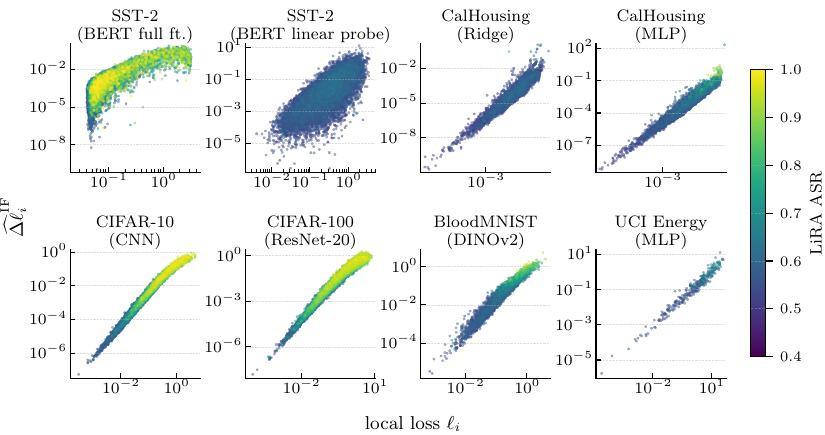}
    \caption{Joint distribution of local loss $\ell_i$ and influence for every training sample of a single target model, coloured by per-sample LiRA ASR, log-log axes.}
    \label{fig:scatter_loss_influence_lira}
\end{figure}

\begin{table}[htbp]
\centering
\small
\setlength{\tabcolsep}{4pt}
\renewcommand{\arraystretch}{1.15}
\caption{Top-10 highest-$|$leverage$|$ training sentences for BERT (linear probe) on SST-2, ranked by Newton (left) and Influence (right) surrogates.}
\label{tab:bert_top_sents_bert_freeze}
\begin{tabular}{@{}c p{0.42\linewidth} p{0.42\linewidth}@{}}
\toprule
Rank & $\widehat{\Delta\ell}^{\mathrm{NS}}_i$ & $\widehat{\Delta\ell}^{\mathrm{IF}}_i$ \\
\midrule
1 & to appearing in this junk that 's tv sitcom material at best & is never quite able to overcome the cultural moat surrounding its l... \\
2 & the filmmakers wisely decided to let crocodile hunter steve irwin d... & spider-man \\
3 & the kind that pretends to be passionate and truthful but is really ... & nicks refuses to let slackers be seen as just another teen movie , ... \\
4 & a supernatural mystery that does n't know whether it wants to be a ... & recycle images and characters that were already tired 10 years ago \\
5 & less the cheap thriller you 'd expect than it is a fairly revealing... & report that the children of varying ages in my audience never cough... \\
6 & played-out & breach gaps in their relationships with their fathers \\
7 & of a callow rich boy who is forced to choose between his beautiful ... & it 's painful to watch witherspoon 's talents wasting away inside u... \\
8 & that collect a bunch of people who are enthusiastic about something... & its drastic iconography \\
9 & : the widowmaker is derivative , overlong , and bombastic -- yet su... & of a callow rich boy who is forced to choose between his beautiful ... \\
10 & ride around a pretty tattered old carousel . & never does `` lilo \& stitch '' reach the emotion or timelessness of... \\
\bottomrule
\end{tabular}
\end{table}

\clearpage
\newpage

\section*{NeurIPS Paper Checklist}

\begin{enumerate}

\item {\bf Claims}
    \item[] Question: Do the main claims made in the abstract and introduction accurately reflect the paper's contributions and scope?
    \item[] Answer: \answerYes{}
    \item[] Justification: The three contributions stated in the introduction, (i)~the closed-form decomposition of the expected leave-one-out gap into the leverage $\bar{h}_{ii}$ and residual $\varepsilon_i$, (ii)~the influence-function and Newton-step surrogates with closed-form last-layer expressions, and (iii)~the top-$k$ recall results on $7/8$ dataset/architecture pairs against LiRA, RMIA and the shadow-model classifier attack, are established respectively in Sections~\ref{sec:theory}, \ref{sec:gen_lin} and \ref{sec:exp_setup}.

\item {\bf Limitations}
    \item[] Question: Does the paper discuss the limitations of the work performed by the authors?
    \item[] Answer: \answerYes{}
    \item[] Justification: The limitations of the work are discussed in the Conclusion.

\item {\bf Theory assumptions and proofs}
    \item[] Question: For each theoretical result, does the paper provide the full set of assumptions and a complete (and correct) proof?
    \item[] Answer: \answerYes{}
    \item[] Justification: All theoretical statements (Lemma~\ref{lem:leverage-residuals}, Proposition~\ref{prop:gap}, Corollary~\ref{cor:optimal_attack}) explicitly state the conditioning on the target point and the Gaussian linear model assumptions. Complete proofs are given in Appendix~\ref{app:sec3_proofs}; the derivations for the influence-function and Newton-step estimators in the generalized-linear case are given in Appendix~\ref{app:sec4_proofs}.

    \item {\bf Experimental result reproducibility}
    \item[] Question: Does the paper fully disclose all the information needed to reproduce the main experimental results of the paper to the extent that it affects the main claims and/or conclusions of the paper (regardless of whether the code and data are provided or not)?
    \item[] Answer: \answerYes{}
    \item[] Justification: Section~\ref{sec:exp_setup} lists the eight datasets/architectures, the LiRA member/non-member split protocol, the number of shadow ($K=200$) and target (16) models, the optimizers used per architecture, and the surrogates being compared. Full hyper-parameters are given in Appendix~\ref{app:arch_hyperparams} and the supplementary code archive contains the launch scripts used to reproduce all reported numbers.

\item {\bf Open access to data and code}
    \item[] Question: Does the paper provide open access to the data and code, with sufficient instructions to faithfully reproduce the main experimental results, as described in supplemental material?
    \item[] Answer: \answerYes{}
    \item[] Justification: All datasets used (CIFAR-10/100, BloodMNIST, SST-2, California Housing, UCI~Energy) are publicly available in a github repository contains training, attack, scoring and figure scripts.

\item {\bf Experimental setting/details}
    \item[] Question: Does the paper specify all the training and test details (e.g., data splits, hyperparameters, how they were chosen, type of optimizer) necessary to understand the results?
    \item[] Answer: \answerYes{}
    \item[] Justification: Section~\ref{sec:exp_setup} reports the data splits (LiRA member/non-member partitioning with per-model seeds), the optimizers used per architecture (AdamW, SGD-momentum, Adam) and the standard recipes followed. The complete list of learning rates, batch sizes, weight decay and number of epochs is provided in Appendix~\ref{app:arch_hyperparams}.

\item {\bf Experiment statistical significance}
    \item[] Question: Does the paper report error bars suitably and correctly defined or other appropriate information about the statistical significance of the experiments?
    \item[] Answer: \answerYes{}
    \item[] Justification: Table~\ref{tab:proxy-recall-at-5pct-lira} reports recall as mean~$\pm$~standard deviation across the 16 independent target models, and Figure~\ref{fig:topk} shows $\pm 1$ standard-deviation bands around the mean of the top-$k$ recall, with the source of variability (random member/non-member splits and target-model initialization) made explicit in Section~\ref{sec:exp_setup}.

\item {\bf Experiments compute resources}
    \item[] Question: For each experiment, does the paper provide sufficient information on the computer resources (type of compute workers, memory, time of execution) needed to reproduce the experiments?
    \item[] Answer: \answerYes{}
    \item[] Justification: Section~\ref{sec:exp_setup} reports that all experiments were run on a V100 GPU cluster, and indicates the compute scale set by training $K=200$ shadow models plus 16 target models per dataset/architecture. Per-job resource and runtime details are listed in the appendix~\ref{app:compute_time}.
    
\item {\bf Code of ethics}
    \item[] Question: Does the research conducted in the paper conform, in every respect, with the NeurIPS Code of Ethics \url{https://neurips.cc/public/EthicsGuidelines}?
    \item[] Answer: \answerYes{}
    \item[] Justification: The work uses only public benchmark datasets and does not involve human subjects, personal data collection, or deployed systems; it conforms with the NeurIPS Code of Ethics. Anonymity was preserved in the submission.

\item {\bf Broader impacts}
    \item[] Question: Does the paper discuss both potential positive societal impacts and negative societal impacts of the work performed?
    \item[] Answer: \answerYes{}
    \item[] Justification: The paper is motivated by the positive impact of enabling cheaper individual privacy-risk auditing, complementing differential-privacy guarantees with a per-sample diagnostic. This technique could not a priori be used directly as an attack, as it relies on accessing the full training dataset to compute the scores.
    
\item {\bf Safeguards}
    \item[] Question: Does the paper describe safeguards that have been put in place for responsible release of data or models that have a high risk for misuse (e.g., pre-trained language models, image generators, or scraped datasets)?
    \item[] Answer: \answerNA{}
    \item[] Justification: The paper does not release new pretrained generative models, scraped datasets, or other high-misuse-risk artifacts. The released code consists of training, attack and scoring scripts on public benchmark datasets.

\item {\bf Licenses for existing assets}
    \item[] Question: Are the creators or original owners of assets (e.g., code, data, models), used in the paper, properly credited and are the license and terms of use explicitly mentioned and properly respected?
    \item[] Answer: \answerYes{}
    \item[] Justification: All datasets (CIFAR-10/100, BloodMNIST, SST-2, California Housing, UCI~Energy) and pretrained backbones (BERT, DINOv2 ViT-small) are publicly available and cited at first use; the attacks we compare against (LiRA, RMIA, the original shadow-model attack) are also cited. We use these assets in accordance with their respective licenses and terms of use.

\item {\bf New assets}
    \item[] Question: Are new assets introduced in the paper well documented and is the documentation provided alongside the assets?
    \item[] Answer: \answerYes{}
    \item[] Justification: The only new asset is the code implementing the computation of Influence and Newton-step scores for the last-layer, both for quadratic and cross-entropy loss. The code and the whole experimental pipeline are included. The supplementary archive includes a README documenting how to reproduce all reported results.

\item {\bf Crowdsourcing and research with human subjects}
    \item[] Question: For crowdsourcing experiments and research with human subjects, does the paper include the full text of instructions given to participants and screenshots, if applicable, as well as details about compensation (if any)?
    \item[] Answer: \answerNA{}
    \item[] Justification: The paper does not involve crowdsourcing or research with human subjects.

\item {\bf Institutional review board (IRB) approvals or equivalent for research with human subjects}
    \item[] Question: Does the paper describe potential risks incurred by study participants, whether such risks were disclosed to the subjects, and whether Institutional Review Board (IRB) approvals (or an equivalent approval/review based on the requirements of your country or institution) were obtained?
    \item[] Answer: \answerNA{}
    \item[] Justification: The paper does not involve crowdsourcing or research with human subjects, so no IRB approval was required.

\item {\bf Declaration of LLM usage}
    \item[] Question: Does the paper describe the usage of LLMs if it is an important, original, or non-standard component of the core methods in this research? Note that if the LLM is used only for writing, editing, or formatting purposes and does \emph{not} impact the core methodology, scientific rigor, or originality of the research, declaration is not required.
    \item[] Answer: \answerNA{}
    \item[] Justification: LLMs are not part of the core methodology. The methods, theoretical results and experiments do not involve LLMs as an important, original, or non-standard component; any LLM use was limited to writing or editing assistance.

\end{enumerate}

\end{document}